\pdftrailerid{}
\documentclass[10pt]{article}

\usepackage[preprint]{tmlr}
\usepackage{amsmath,amssymb,amsthm,mathtools}
\usepackage{aliascnt}
\usepackage{booktabs,array}
\usepackage{graphicx}
\usepackage{xcolor}
\usepackage{float}
\usepackage{url}
\usepackage[colorlinks=true,linkcolor=blue!60!black,citecolor=blue!60!black,urlcolor=blue!60!black]{hyperref}
\usepackage[capitalize,nameinlink]{cleveref}
\hypersetup{
  pdftitle={Finite-Probe Total-Variation Certificates for Finite-Basis Drifting Models},
  pdfauthor={Sam Andersson and Ricky Molén},
  pdfsubject={Finite-probe observability and conditioning in drifting models}
}

\newtheorem{theorem}{Theorem}[section]

\newaliascnt{lemma}{theorem}
\newtheorem{lemma}[lemma]{Lemma}
\aliascntresetthe{lemma}

\newaliascnt{proposition}{theorem}
\newtheorem{proposition}[proposition]{Proposition}
\aliascntresetthe{proposition}

\newaliascnt{corollary}{theorem}
\newtheorem{corollary}[corollary]{Corollary}
\aliascntresetthe{corollary}

\newaliascnt{conjecture}{theorem}
\newtheorem{conjecture}[conjecture]{Conjecture}
\aliascntresetthe{conjecture}

\theoremstyle{definition}
\newaliascnt{definition}{theorem}
\newtheorem{definition}[definition]{Definition}
\aliascntresetthe{definition}

\newaliascnt{example}{theorem}
\newtheorem{example}[example]{Example}
\aliascntresetthe{example}

\theoremstyle{remark}
\newaliascnt{remark}{theorem}
\newtheorem{remark}[remark]{Remark}
\aliascntresetthe{remark}

\crefname{theorem}{Theorem}{Theorems}
\crefname{proposition}{Proposition}{Propositions}
\crefname{lemma}{Lemma}{Lemmas}
\crefname{corollary}{Corollary}{Corollaries}
\crefname{definition}{Definition}{Definitions}
\crefname{remark}{Remark}{Remarks}
\crefname{example}{Example}{Examples}
\crefname{conjecture}{Conjecture}{Conjectures}
\Crefname{theorem}{Theorem}{Theorems}
\Crefname{proposition}{Proposition}{Propositions}
\Crefname{lemma}{Lemma}{Lemmas}
\Crefname{corollary}{Corollary}{Corollaries}
\Crefname{definition}{Definition}{Definitions}
\Crefname{remark}{Remark}{Remarks}
\Crefname{example}{Example}{Examples}
\Crefname{conjecture}{Conjecture}{Conjectures}
\crefname{appendix}{Appendix}{Appendices}
\Crefname{appendix}{Appendix}{Appendices}

\newcommand{\R}{\mathbb{R}}
\newcommand{\E}{\mathbb{E}}
\newcommand{\PP}{\mathbb{P}}
\newcommand{\norm}[1]{\lVert #1 \rVert}
\newcommand{\Norm}[1]{\left\lVert #1 \right\rVert}
\newcommand{\abs}[1]{\lvert #1 \rvert}
\newcommand{\vecop}{\operatorname{vec}}
\newcommand{\rk}{\operatorname{rank}}

\newcommand{\smin}{\sigma_{\min}}
\newcommand{\smax}{\sigma_{\max}}
\newcommand{\lmin}{\lambda_{\min}}
\newcommand{\lmax}{\lambda_{\max}}
\newcommand{\Vhat}{\widehat{V}}

\newcommand{\TV}{\mathrm{TV}}
\newcommand{\Leb}{\mathrm{Leb}}
\newcommand{\dd}{\,d}
\newcommand{\ip}[2]{\langle #1, #2 \rangle}

\newcolumntype{L}[1]{>{\raggedright\arraybackslash}p{#1}}

\title{Finite-Probe Total-Variation Certificates for Finite-Basis Drifting
Models}
\author{\name Sam Andersson$^{1}$\thanks{Corresponding author: Sam Andersson,
Department of Clinical Neuroscience, Karolinska Institutet, Stockholm, Sweden.
Email: \href{mailto:Sam.andersson@ki.se}{Sam.andersson@ki.se}.}
\qquad
\name Ricky Molén$^{2}$\\[-0.15em]
\addr{$^{1}$Department of Clinical Neuroscience, Karolinska Institutet,
Stockholm, Sweden}\\
\addr{$^{2}$Department of Computational Science and Technology, KTH Royal
Institute of Technology, Stockholm, Sweden}}

\begin{document}

\maketitle

\begin{abstract}
Drifting objectives compare a target and model distribution through a vector
field observed noisily at finitely many locations. We ask what distributional
conclusion such a frozen measurement system warrants. For integrable
antisymmetric interactions and absolutely continuous laws in a declared
finite density basis, the unnormalized sampled numerator satisfies
$\operatorname{vec}(V_X)=Mc$, where $c$ is an antisymmetric mismatch and $M$
is probe-dependent. This identity yields an a posteriori total-variation (TV)
upper confidence bound accounting for held-out field noise,
estimated-operator error, and externally validated $L^1$ residual radii around
normalized density approximants in the span; a nonpositive observability
margin returns the trivial TV bound and abstains. The audit recomputes this
numerator from held-out samples; a normalized drift statistic requires a
separate joint numerator--denominator analysis. For Gaussian-RBF interactions,
a global envelope supports distribution-free and empirical-Bernstein radii
without truncation, with companion bounds for the Laplace similarity in the
original drifting objective. We characterize random-probe observability by a
population Gram matrix, identify rank and symmetry degeneracies, and prove
large-bandwidth collapse toward mean matching. Synthetic studies exercise
Gaussian and Laplace numerators, separately prespecified bounded-vector and
variance-adaptive radii, Monte Carlo-calibrated operators, nonzero residual
radii around normalized finite-basis approximants, outward-rounded
observability bounds, and designed abstention. A joint basis-size/dimension
stress path extends evaluation through $m=8$. The result is a conditional
diagnostic for a finite density class, or for normalized finite-basis density
approximants with external residual radii, not a universal guarantee from
small training drift.
\end{abstract}

\section{Introduction}\label{sec:intro}

Drifting models train a one-step generator using a
distribution-dependent vector field $V_{p,q}$ that compares a target law
$p$ with the generator law $q$ \citep{deng2026drifting}.  At population
equilibrium, $p=q$ implies $V_{p,q}=0$.  The question relevant to finite
computation is harder: if held-out estimates of the field are small at $N$
fixed locations, how close must $q$ be to $p$?  Population identifiability
alone does not answer this question.  Finite probes can miss mismatch
directions, ill-conditioning can amplify sampling error, and a finite
representation of the distributions introduces residual error.

We study distributional rather than parameter identifiability.  Different
generator parameters may induce the same pushforward law; our object is the
pair $(p,q)$. The main results concern absolutely continuous probability
laws whose densities lie in a shared finite basis, or are supplied with
independently valid $L^1$ residual radii around normalized
probability-density elements of that span. This excludes, without further
work, singular generator pushforwards and empirical measures. For an
absolutely integrable antisymmetric interaction, within this scope the
sampled numerator field at
$X=(x_1,\ldots,x_N)$ takes the form
\begin{equation}
  \vecop(V_X)=Mc. \label{eq:obs-intro}
\end{equation}
Here $M$ is determined by the interaction kernel, basis, and probes, while
$c=a\wedge b$ records the antisymmetric coefficient mismatch between $p$ and
$q$.  This representation exposes what raw drift omits: an observability
margin.  If $M$ is rank deficient, some formal mismatch directions are
invisible.  If it is full rank but poorly conditioned, a small sampled field
can coexist with a large mismatch.

Our main result turns \eqref{eq:obs-intro} into a calibrated, a posteriori
audit.  Let $\Vhat_X$ be evaluated on a fresh held-out batch, let
$\varepsilon_V$ control its sampling error, let $\varepsilon_M$ control an
estimated observation matrix, and let $R_m$ collect representation
residuals.  The resulting master bound has the form
\begin{equation}
  \norm{c_m}_2
  \le
  \frac{\norm{\vecop(\Vhat_X)}_2+\varepsilon_V+\norm{R_m}_2}
       {\underline\sigma-\varepsilon_M},
  \qquad \underline\sigma-\varepsilon_M>0,
  \label{eq:master-intro}
\end{equation}
where $\underline\sigma$ is a certified lower bound on the smallest
singular value of the computed operator.  Throughout,
$\smin(A):=\inf_{\norm{u}_2=1}\norm{Au}_2$; it is therefore zero whenever
$A$ is noninjective, including whenever it is wide.  A sharp
exterior-product coefficient identity converts this inequality to total
variation within the basis; explicit $L^1$ residual terms extend it when
$p$ and $q$ have normalized probability-density approximants in the declared
span. If the denominator is not positive, the valid output is abstention,
not a numerical claim.

The certificate is particularly explicit for a Gaussian-RBF mean-shift
interaction, a bounded analytic variant of the drifting numerator.  Although
its inputs may have unbounded support, the complete interaction obeys the
global envelope
$\norm{K_\tau}_2\le\sqrt{\tau/e}$.  Consequently, both distribution-free
and empirical-Bernstein radii apply to untruncated Gaussian samples.
For the Laplace similarity used by \citet{deng2026drifting}, the companion
envelope is $\norm{K_\tau}_2\le\tau/e$.
\Cref{cor:operational} combines the drift, matrix, and residual error budgets
into a held-out total-variation upper confidence bound and a corresponding
equivalence test.

The exact finite-basis implication underlying \eqref{eq:obs-intro} appears
in Appendix~C.1 of \citet{deng2026drifting}; generic singular-value,
concentration, and matrix-perturbation inequalities are also standard.  We
do not claim these ingredients separately.  Our contribution is their
calibrated composition into a finite-observation distributional certificate,
together with drift-specific feasible geometry, probe-law analysis, and
explicit abstention conditions.  This addresses a different question from
recent work on what the
population field represents, which dynamics it induces, or whether its full
continuum of values is identifying
\citep{weber2023scoredifference,turan2026secretly,lai2026unified,
cao2026gradientflow,franz2026nonconservative,lee2026identifiability,
balasubramanian2026finiteparticle}.

The argument proceeds in three stages. \Cref{prop:stability} isolates the
conditioning principle, \cref{thm:endtoend} adds all observation and
representation errors in one inverse inequality, and
\cref{cor:operational} turns that inequality into a held-out TV audit with
abstention. The remaining results calibrate its radii, analyze probe design,
and expose the predicted failure regimes.

The paper makes three contributions.

\begin{enumerate}
\item \textbf{An end-to-end held-out certificate.}
We propagate empirical-field error, operator-calibration error, and
finite-model residuals through the observability denominator.  The exact
exterior-product coefficient identity and resulting coefficient bridge yield
a total-variation upper confidence bound with an explicit abstention rule.
For Gaussian-RBF interactions we derive global,
distribution-free and variance-adaptive radii without sample truncation, and
we give a Laplace-kernel counterpart. Residual radii around normalized
finite-basis density approximants are declared external inputs; without them
the output is a sensitivity analysis, not an unconditional full-distribution
certificate.

\item \textbf{Probe observability and formal-versus-feasible geometry.}
The population Gram matrix $\Gamma(\nu)$ governs random-probe conditioning.
For equal-covariance isotropic Gaussian bases with pairwise-distinct component
means and pairwise-distinct pair midpoints, we verify the criterion
analytically and show how to use
$\gamma(\nu)=\lambda_{\min}(\Gamma(\nu))$ as a probe-design objective.
Because $c=a\wedge b$ is a decomposable bivector, ambient null
directions need not correspond to pairs of probability distributions; we
separate these formal and feasible notions.

\item \textbf{Failure regimes and empirical validation.}
Directional-rank and midpoint degeneracies expose blind directions, while a
large-bandwidth theorem shows convergence toward first-moment comparison and,
under finite third moments for the Gaussian kernel, an $O(1/\tau)$ rate on
fixed probes.  Controlled experiments test held-out coverage,
certificate tightness, probe-law design, bandwidth collapse, and numerical
rank boundaries.  A three-component experiment compares separately
prespecified bounded-vector and empirical-Bernstein audits for Gaussian and
Laplace numerators while exercising operator calibration, a nonzero residual
budget, an outward-rounded Gram lower bound, the final TV output, and
abstention.  A joint regular-simplex path through $m=8$ exposes the
finite-sample cost of increasing the basis and observation system.  A
separate cross-bandwidth study
shows that neither raw drift nor its radius-free conditioned plug-in exhibits
a bandwidth-invariant Wasserstein calibration in these benchmark designs.
\end{enumerate}

We do not claim a state-of-the-art benchmark for generative models or a
universal convergence theorem for arbitrary drifting fields.  Our results
concern a declared finite model class.  They extend to a larger class only
when normalized finite-basis density approximants and valid residual
envelopes are supplied.  The probes and audit data must be independent of
the choices made during training; tuning the generator, bandwidth, basis, or
probes on the same batch would require additional uniform or sequential
corrections.  The audit uses fresh samples from $p$ and $q$, but it is not
intended as a universal two-sample procedure.  It asks what can be inferred
from the unnormalized numerator of a particular drifting measurement system
when that numerator is recomputed on held-out data.  Under these conditions,
small held-out numerator drift certifies closeness only if the measurement
operator is sufficiently well conditioned and every stated error input is
valid.

\section{Related Work}\label{sec:related}

\paragraph{Drifting fields and population identifiability.}
\citet{deng2026drifting} introduced drifting as a one-step generative
paradigm in which a generator is regressed toward targets moved by a
distribution-dependent field.  Gaussian-smoothed score-difference transport
and a transport-and-regression update appeared earlier in
\citet{weber2023scoredifference}.  Subsequent work connects drifting to
smoothed score differences, variational transport, and Wasserstein gradient
flows \citep{turan2026secretly,lai2026unified,cao2026gradientflow,
gretton2026wgf}; \citet{franz2026nonconservative} analyze when the usual
normalization is nonconservative.  Population identifiability is established
for several kernel variants \citep{lee2026identifiability,esteban2026kernel,
he2026sinkhorn}, while \citet{balasubramanian2026finiteparticle} studies
finite-particle dynamics.  These results clarify the continuum field and its
dynamics.  We instead invert noisy observations at an external, finite probe
set.

The exact finite-basis implication used here is already present in
Appendix~C.1 of \citet{deng2026drifting}: linear independence of
probe-induced interaction vectors makes zero sampled drift imply equality
inside the basis.  Likewise, \citet{falahati2026driftxpress} analyze a
Nystr\"om-type projection residual for computational acceleration.  Our
residual has a different role---it quantifies approximation of the
underlying laws---and our output is a finite-sample upper confidence bound,
not an exact implication or acceleration guarantee.

More precisely, \citet{lai2026unified} control deviations between practical
finite-pool normalized updates and population score-like fields, whereas our
inverse step asks what distributional upper bound follows from a frozen
finite measurement system.  \citet{lee2026identifiability} study
continuum-field identifiability and weak stability for broad Borel laws under
tightness conditions.  Their distribution-free population scope is stronger
than ours; conversely, it does not yield a TV upper confidence bound from
finitely many noisy probes.  Our TV conclusion requires a known finite
density class or externally validated $L^1$ residual radii around normalized
density approximants in the declared span.

\Cref{tab:related-comparison} summarizes these distinctions.
\begin{table}[t]
\centering
\scriptsize
\setlength{\tabcolsep}{3.5pt}
\begin{tabular}{@{}L{2.15cm}L{2.15cm}L{2.65cm}L{5.25cm}@{}}
\toprule
Work & Observations & Distribution class & Statistical output \\
\midrule
\citet{deng2026drifting}, App.~C.1
& exact finite probes
& exact finite density basis
& equality implication under probe-induced linear independence; no noisy
  upper bound \\
\citet{lai2026unified}
& normalized finite pools and population fields
& smooth population models
& implementation-to-population deviation and dynamics; no TV certificate \\
\citet{lee2026identifiability}
& full continuum field
& broad Borel probability laws
& population identifiability and weak stability under tightness \\
Kernel equivalence tests \citep{liu2026equivalence}
& empirical kernel witnesses
& nonparametric
& equivalence in MMD or Stein discrepancy, not TV \\
This work
& noisy frozen finite probes
& finite density class or normalized finite-span approximants with external
  $L^1$ radii
& one-sided TV upper confidence bound with operator error and abstention \\
\bottomrule
\end{tabular}
\caption{Closest distinctions.  ``Finite'' refers to the measurement set,
not to the support of the distributions.}
\label{tab:related-comparison}
\end{table}

\paragraph{Finite witnesses, testing, and certifying closeness.}
Characteristic-kernel MMD and kernelized Stein discrepancies, under their
respective kernel and regularity conditions, provide equality-determining
population witnesses with empirical tests
\citep{gretton2008kernel,gretton2012kernel,sriperumbudur2010hilbert,
liu2016ksd}.  Analytic mean embeddings and finite-set Stein tests
show that finitely many random witness evaluations can separate a fixed
alternative, and optimize locations for power
\citep{chwialkowski2015analytic,jitkrittum2016features,
jitkrittum2017linear}.  Such detection results do not by themselves turn
failure to reject into a closeness guarantee.  Kernel equivalence tests
reverse the hypotheses and certify a margin in MMD or Stein discrepancy
\citep{liu2026equivalence}.  Our tradeoff is different: after restricting to
a finite or explicitly approximated model class, we obtain a one-sided upper
bound in total variation and pay explicitly for drift-observation
conditioning.  \Cref{cor:operational-equivalence} states the induced
TV-equivalence test.  Because the audit itself draws fresh samples from both
laws, direct basis-coefficient estimation or a conventional structured
two-sample procedure may be statistically tighter when the basis is known.
We do not claim minimax efficiency; the question here is what can be inferred
through the specified drifting measurement operator.

\paragraph{Conditioning, design, and structured inverse problems.}
Evaluating a field at probes selects a finite collection of moment
restrictions.  The distinction between valid moments and strong
identification is classical in generalized method of moments
\citep{hansen1982gmm,stockwright2000weak}.  The matrix
$\Gamma(\nu)=\E_\nu[G(X)^\top G(X)]$ is an information matrix, so maximizing
its smallest eigenvalue is the classical E-optimal design criterion
\citep{kiefer1974design,allenzhu2017design}.  We use this machinery to expose
the conditioning geometry induced by drifting kernels and finite density
bases.  A drift-specific feature is that $c=a\wedge b$ is a decomposable
bivector rather than an arbitrary element of
$\R^{\binom{m}{2}}$, connecting the formal-versus-feasible distinction to
structured and algebraic inverse problems
\citep{chandrasekaran2012geometry,breiding2023algebraic}.

Prior work therefore supplies both the exact finite-basis starting point and
the generic analysis tools.  The distinct object here is an operational TV
upper bound that simultaneously accounts for held-out field noise,
observation-operator estimation, finite-model residuals, and probe
conditioning.  The qualifier ``operational'' refers to the explicit audit
inputs and abstention rule, not to an ability to infer unknown approximation
radii or to certify from a normalized training statistic without an
additional denominator analysis.

\section{Preliminaries}\label{sec:prelim}

Throughout, we use the notation summarized in \cref{tab:notation}.

\subsection{Distributions and pushforwards}\label{sec:pushforwards}

The sample space is $\R^d$. The target distribution is denoted by $p$. The
generated distribution is denoted by $q$, or by $q_\theta$ when we want to
emphasize that it is induced by a measurable generator $f_\theta$. If
$\epsilon \sim p_\epsilon$ and $x = f_\theta(\epsilon)$, then
\[
  q_\theta = (f_\theta)_\# p_\epsilon
\]
means that $q_\theta$ is the law of $f_\theta(\epsilon)$. Equivalently, for
every measurable set $A \subseteq \R^d$,
\[
  q_\theta(A) = p_\epsilon\big( f_\theta^{-1}(A) \big).
\]
Such a pushforward need not have a Lebesgue density.  Assumptions (A1) and
the approximation results below apply only when $p$ and $q$ are absolutely
continuous and have densities in, or with certified $L^1$ distance to, the
declared finite span; for the approximation result, the finite-span
approximants must themselves be normalized probability densities. Singular
pushforwards and empirical measures are outside the present TV theorem.

\subsection{Drift fields}\label{sec:driftfields}

A drift field is a vector-valued map
\[
  V_{p,q} : \R^d \to \R^d.
\]
The input $x \in \R^d$ is the location at which the field is evaluated.
The output $V_{p,q}(x) \in \R^d$ is the proposed movement direction at $x$.
We use $p$ for the target distribution and $q$ for the current generated
distribution.

Let $Y^+ \sim p$ and $Y^- \sim q$ be independent random variables. A
general interaction-kernel drift has the form
\begin{equation}
  V_{p,q}(x) = \E\big[ K(x, Y^+, Y^-) \big], \label{eq:kernel-drift}
\end{equation}
where
\[
  K : \R^d \times \R^d \times \R^d \to \R^d
\]
is a vector-valued interaction kernel. If $p$ and $q$ admit densities, then
independence gives
\begin{equation}
  V_{p,q}(x)
  = \int\!\!\int K(x, y^+, y^-)\, p(y^+) q(y^-) \dd y^+ \dd y^-.
  \label{eq:drift-int}
\end{equation}
All vector expectations and integrals in the paper are Bochner integrals.
The measurability and absolute-integrability conditions needed for the
general finite-basis results are collected in (A0) below; the global
Gaussian and Laplace envelopes verify them whenever the basis functions are
in $L^1$.

\begin{definition}[Antisymmetric interaction kernel]\label{def:antisym}
The interaction kernel $K$ is antisymmetric in its sample arguments if
\[
  K(x, y^+, y^-) = -K(x, y^-, y^+)
\]
for every $x, y^+, y^- \in \R^d$.
\end{definition}

If $K$ is antisymmetric and the integrals are finite, then $p = q$ implies
$V_{p,q}(x) = 0$ for every $x$ (\cref{app:antisym-equilibrium}). The
converse is the main object of study.

\subsection{Probe locations and sampled drift}\label{sec:probes}

A probe list is an ordered tuple
\[
  X=(x_1,\dots,x_N)\in(\R^d)^N.
\]
Repeated locations are allowed; when distinctness is irrelevant, we also
refer informally to the list as a probe set.
The sampled drift matrix is
\begin{equation}
  V_X = \big[\, V_{p,q}(x_1) \;\; V_{p,q}(x_2) \;\cdots\; V_{p,q}(x_N)
  \,\big] \in \R^{d \times N}. \label{eq:VX}
\end{equation}
The vectorization operator $\vecop$ stacks the entries of a matrix into a
single vector in a fixed order. Therefore $\vecop(V_X) \in \R^{dN}$. The
precise stacking convention does not matter as long as it is used
consistently.
For the same convention, define the stacked interaction vector
\begin{equation}
  K_X(y^+,y^-)
  :=\vecop\big(K(x_1,y^+,y^-)\;\cdots\;K(x_N,y^+,y^-)\big)
  \in\R^{dN}.
  \label{eq:KX}
\end{equation}

\subsection{Mean-shift-style drift}\label{sec:meanshift}

The main concrete example in this paper is the unnormalized
mean-shift-style interaction
\begin{equation}
  K(x, y^+, y^-) = k(x, y^+)\, k(x, y^-)\, (y^+ - y^-),
  \label{eq:meanshift}
\end{equation}
where $k : \R^d \times \R^d \to \R$ is a scalar similarity kernel. This
kernel is antisymmetric because swapping $y^+$ and $y^-$ changes
$y^+ - y^-$ to $-(y^+ - y^-)$ while leaving the scalar product
$k(x, y^+) k(x, y^-)$ unchanged.

Our concrete benchmark uses the Gaussian-RBF similarity
\begin{equation}
  k_\tau(x,y)=\exp\!\left(-\frac{\norm{x-y}_2^2}{\tau}\right),
  \qquad \tau>0. \label{eq:rbf}
\end{equation}
Its exponential factor makes the interaction globally bounded even when the
sample distributions have unbounded support.

The original drifting objective of \citet{deng2026drifting} instead uses the
Laplace similarity
\begin{equation}
  \ell_\tau(x,y)=\exp\!\left(-\frac{\norm{x-y}_2}{\tau}\right),
  \qquad \tau>0. \label{eq:laplace}
\end{equation}
The finite-basis observation equation is agnostic to the kernel form within
the absolutely integrable antisymmetric interaction class. The Gaussian
choice supplies the closed forms used in our experiments; the next two
lemmas give rigorous global envelopes for both choices.

\begin{lemma}[Global Gaussian-RBF interaction envelope]
\label{lem:rbf-envelope}
Let $K_\tau$ be the interaction \eqref{eq:meanshift} with $k=k_\tau$. Then,
for every $x,y^+,y^-\in\R^d$,
\begin{equation}
  \norm{K_\tau(x,y^+,y^-)}_2\le\sqrt{\frac{\tau}{e}}.
  \label{eq:rbf-envelope}
\end{equation}
Consequently each coordinate is bounded in absolute value by
$B_\infty=\sqrt{\tau/e}$, and for any probe set of size $N$,
\begin{equation}
  \norm{K_X(y^+,y^-)}_2\le B_{N,\tau}
  :=\sqrt{\frac{N\tau}{e}}. \label{eq:stacked-envelope}
\end{equation}
Thus (A5) and (A7) of \cref{sec:observability} hold with explicit constants,
without truncating the sample distributions.
\end{lemma}

The proof is deferred to \cref{app:proof-rbf-envelope}.

\begin{lemma}[Global Laplace interaction envelope]
\label{lem:laplace-envelope}
Let $K_\tau$ be the interaction \eqref{eq:meanshift} with
$k=\ell_\tau$. Then, for every $x,y^+,y^-\in\R^d$,
\begin{equation}
  \norm{K_\tau(x,y^+,y^-)}_2\le\frac{\tau}{e}.
  \label{eq:laplace-envelope}
\end{equation}
Consequently each coordinate is bounded by $B_\infty=\tau/e$, and the
stacked interaction over $N$ probes is bounded by
$B_{N,\tau}^{\mathrm{Lap}}=\sqrt{N}\tau/e$.
Thus (A5) and (A7) of \cref{sec:observability} hold with these constants.
\end{lemma}

The proof is deferred to \cref{app:proof-laplace-envelope}.

\begin{remark}[Normalized drift]\label{rem:normalized}
A natural normalized population field, closely related to the
normalizations used in drifting implementations, is
\[
  \bar V_{p,q}(x) := \frac{V_{p,q}(x)}{Z(x)},
  \qquad
  Z(x) := \E\big[ k(x, Y^+)\, k(x, Y^-) \big]
        = \E\big[ k(x, Y^+) \big]\, \E\big[ k(x, Y^-) \big],
\]
where the factorization uses independence of $Y^+$ and $Y^-$. At any probe
$x$ with $0 < \abs{Z(x)} < \infty$, one has
$\bar V_{p,q}(x) = 0$ if and only if
$V_{p,q}(x) = 0$; hence, whenever the normalizers are finite and nonzero at
every probe, the sampled zero-drift conditions for the normalized and
unnormalized fields coincide. Consequently, every exact-identifiability
conclusion below remains valid when the premise $V_X=0$ is replaced by
$\bar V_X=0$. Population magnitude statements also transfer: from
$V_{p,q}(x_\ell) = Z(x_\ell)\, \bar V_{p,q}(x_\ell)$,
\[
  \norm{\vecop(V_X)}_2
  \;\le\; \Big( \max_{1 \le \ell \le N} \abs{Z(x_\ell)} \Big)\,
  \norm{\vecop(\bar V_X)}_2 ,
\]
where $\bar V_X$ denotes the normalized sampled drift matrix. For the
Gaussian-RBF kernel of \eqref{eq:rbf}, $0 < k_\tau \le 1$ pointwise, so
$0 < Z(x) \le 1$ automatically (the strict positivity follows because the
expectation of a strictly positive function is strictly positive).
For a general similarity, the factorization additionally requires
$\E\abs{k(x,Y^+)}<\infty$ and $\E\abs{k(x,Y^-)}<\infty$.

Practical implementations use several normalization and finite-batch
variants, which need not coincide exactly with this population ratio
\citep{lai2026unified,franz2026nonconservative}. The results below therefore
state explicitly whether they concern the numerator field or a
ratio-normalized version.

Empirical normalized drift needs an additional joint
numerator--denominator analysis. If an implementation uses a ratio estimator
\[
  \widehat{\bar V}_{p,q}(x) = \frac{\widehat N(x)}{\widehat Z(x)},
\]
then the sampling error includes both numerator error and denominator
fluctuation. The empirical certificates below apply directly to the
numerator field. To apply them to the ratio estimator, assume for example
\[
  \inf_{1\le \ell\le N} Z(x_\ell) \ge z_{\min} > 0,
  \qquad
  \max_{1\le \ell\le N}\abs{\widehat Z(x_\ell)-Z(x_\ell)}
  \le z_{\min}/2.
\]
Under such an event, ratio perturbation bounds can convert numerator
certificates into normalized certificates with constants depending on
$z_{\min}^{-1}$, but this paper does not calibrate that joint event.  The
implemented audit therefore recomputes the cross-multiplied numerator on
held-out samples; a small normalized training statistic by itself is not an
input to \cref{cor:operational}.
\end{remark}

\begin{table}[H]
\centering
\scriptsize
\setlength{\tabcolsep}{4pt}
\begin{tabular}{L{1.8cm}L{2.4cm}L{4.9cm}L{4.0cm}}
\toprule
Symbol & Space & Meaning & Defined \\
\midrule
$p,\,q,\,q_\theta$ & laws; densities under (A1) & target / generated distribution & \cref{sec:pushforwards} \\
$V_{p,q}$ & $\R^d \to \R^d$ & population drift field & \cref{sec:driftfields} \\
$K$ & $\R^{3d} \to \R^d$ & interaction kernel & \eqref{eq:kernel-drift} \\
$k_\tau,\ell_\tau$ & $\R^{2d} \to \R$ & Gaussian-RBF / Laplace similarity & \eqref{eq:rbf}, \eqref{eq:laplace} \\
$X,\ N$ & $(\R^d)^N$;\ $\mathbb N$ & ordered probe list; length & \cref{sec:probes} \\
$V_X$ & $\R^{d\times N}$ & population sampled drift & \eqref{eq:VX} \\
$K_X$ & $\R^{2d}\to\R^{dN}$ & interaction vector stacked over the probes & \eqref{eq:KX} \\
$\Vhat_X$ & $\R^{d\times N}$ & empirical sampled drift ($n$ pairs) & \eqref{eq:empdrift} \\
$\phi_1,\dots,\phi_m$ & functions on $\R^d$ & basis of the model class & \eqref{eq:finite-basis} \\
$a,\ b$ & $\R^m$ & basis coefficients of $p$, $q$ & \eqref{eq:finite-basis} \\
$c$;\ $r$ & $\R^r$;\ $r=\binom{m}{2}$ & mismatch, $c_{ij}=a_ib_j-a_jb_i$ & \eqref{eq:cij}, \eqref{eq:cdef} \\
$U_{ij}$ & $\R^{d\times N}$ & pair response at the probes & \eqref{eq:Uij} \\
$M$ / $M_m$;\ $\widehat M_m$ & $\R^{dN\times r}$ & observation operator; its computed estimate & \eqref{eq:Mdef}; \cref{sec:estM} \\
$\varepsilon_V$;\ $\varepsilon_M$ & $\R_{\ge 0}$ & drift / matrix estimation error & \eqref{eq:master-errors} \\
$B_{N,\tau}$ & $\R_{>0}$ & stacked Gaussian-RBF envelope $\sqrt{N\tau/e}$ & \eqref{eq:stacked-envelope} \\
$R_m$ & $\R^{dN}$ & basis-approximation residual & \eqref{eq:approx-decomp} \\
$\rho_p,\rho_q$;\ $\overline R_m$ & $\R_{\ge0}$ & density / drift residual radii & \eqref{eq:residual-envelope}; \eqref{eq:Rbar} \\
$\eta$ & $\R^m$, $\eta_i = \int \phi_i$ & basis mass vector & \cref{lem:bridge} \\
$\beta_\phi$ & $\R_{\ge0}$ & mismatch-to-TV conversion constant & \eqref{eq:betaphi} \\
$g_\alpha$ & $\R^d \to \R^d$ & pair witness, $\alpha \in P$ & \eqref{eq:pairwitness} \\
$G(x)$ & $\R^{d\times r}$ & probe block & \cref{sec:randprobe} \\
$\nu$;\ $\Gamma(\nu)$;\ $\gamma(\nu)$ & law on $\R^d$; $\R^{r\times r}$; $\R_{\ge 0}$ & probe law; observability Gram matrix; $\lmin(\Gamma(\nu))$ & \eqref{eq:Gamma} \\
$L$;\ $B_\infty$;\ $B_X$ & $\R_{>0}$ & bounds in (A9); (A5); (A7) & \cref{sec:randprobe}; \cref{lem:drift-radii}; \cref{prop:residual} \\
$\Delta\mu$ & $\R^d$ & $\E_p Y - \E_q Y$ & \cref{sec:collapse} \\
$C(x,p,q)$ & $\R_{\ge 0}$ & finite-bandwidth collapse constant & \cref{lem:taurate} \\
$\Delta$;\ $U_c$;\ $U_{\mathrm{TV}}$ & $\R$; $[0,+\infty]$; $[0,1]$ & audit margin and upper bounds & \eqref{eq:Delta-audit}; \eqref{eq:operational-bounds} \\
\bottomrule
\end{tabular}
\caption{Notation. Population quantities ($V_X$, $M$) are distinguished
from their empirical counterparts ($\Vhat_X$, $\widehat M_m$) throughout; every
certificate states explicitly which object it bounds ($c$, $c_m$, or a
distributional metric via \cref{cor:tv}).}
\label{tab:notation}
\end{table}

\section{Finite-Basis Drift Observability}\label{sec:observability}

The central calculation rewrites the zero-drift condition at finitely many
probes as a finite-dimensional linear observation equation.

\subsection{Finite-dimensional model class}\label{sec:modelclass}

Fix $m\ge2$. Let $\phi_1, \dots, \phi_m$ be linearly independent basis functions on
$\R^d$. Assume that $p$ and $q$ belong to their finite span:
\begin{equation}
  p(y) = \sum_{i=1}^{m} a_i \phi_i(y),
  \qquad
  q(y) = \sum_{i=1}^{m} b_i \phi_i(y). \label{eq:finite-basis}
\end{equation}
The coefficient vectors are $a = (a_1, \dots, a_m)^\top \in \R^m$ and
$b = (b_1, \dots, b_m)^\top \in \R^m$. The coefficients are assumed to make
$p$ and $q$ valid probability densities.

Starting from \eqref{eq:drift-int} and substituting
\eqref{eq:finite-basis} gives
\begin{equation}
  V_{p,q}(x)
  = \sum_{i=1}^m \sum_{j=1}^m a_i b_j
    \int\!\!\int K(x, y^+, y^-)\, \phi_i(y^+) \phi_j(y^-)
    \dd y^+ \dd y^-. \label{eq:drift-expansion}
\end{equation}
The finite sums allow the expansion by distributivity and linearity of
integration, assuming the displayed integrals are finite.

For each pair $(i,j)$, define $U_{ij} \in \R^{d \times N}$ by its
$\ell$-th column:
\begin{equation}
  U_{ij}[:, \ell]
  = \int\!\!\int K(x_\ell, y^+, y^-)\, \phi_i(y^+) \phi_j(y^-)
    \dd y^+ \dd y^-, \qquad \ell = 1, \dots, N. \label{eq:Uij}
\end{equation}
Evaluating \eqref{eq:drift-expansion} at every probe gives
\begin{equation}
  V_X = \sum_{i=1}^m \sum_{j=1}^m a_i b_j\, U_{ij}. \label{eq:VX-sum}
\end{equation}

\subsection{Antisymmetric mismatch coordinates}\label{sec:mismatch}

Assume the kernel is antisymmetric in the sense of \cref{def:antisym}.
Then $U_{ji} = -U_{ij}$ for all $i, j$
(\cref{app:U-antisym}), and $U_{ii} = 0$. Grouping the $(i,j)$ and $(j,i)$
terms in \eqref{eq:VX-sum} gives, for $i < j$,
\[
  a_i b_j U_{ij} + a_j b_i U_{ji} = (a_i b_j - a_j b_i)\, U_{ij}.
\]
Define the antisymmetric coefficient mismatch
\begin{equation}
  c_{ij} = a_i b_j - a_j b_i, \qquad 1 \le i < j \le m. \label{eq:cij}
\end{equation}
Then
\begin{equation}
  V_X = \sum_{1 \le i < j \le m} c_{ij}\, U_{ij}. \label{eq:VX-anti}
\end{equation}
Let
\[
  P = \{(i,j) : 1 \le i < j \le m\}, \qquad r = \abs{P} = \binom{m}{2}.
\]
Fix an ordering $P = \{(i_1, j_1), \dots, (i_r, j_r)\}$. Define
\begin{equation}
  M = \big[\, \vecop(U_{i_1 j_1}) \;\; \vecop(U_{i_2 j_2}) \;\cdots\;
      \vecop(U_{i_r j_r}) \,\big] \in \R^{dN \times r},
  \label{eq:Mdef}
\end{equation}
and
\begin{equation}
  c = \big( c_{i_1 j_1}\; c_{i_2 j_2} \cdots c_{i_r j_r} \big)^\top
  \in \R^{r}. \label{eq:cdef}
\end{equation}
Vectorizing \eqref{eq:VX-anti} yields the observation equation
\begin{equation}
  \vecop(V_X) = M c. \label{eq:obs}
\end{equation}

We collect here the assumptions used throughout the paper, numbered
globally. Their objects and use sites are defined at the referenced
locations.
\begin{itemize}
\item[(A0)] The basis functions and $K$ are Borel measurable and, for every
  $i,j,\ell$,
  \[
    \int\!\!\int
      \norm{K(x_\ell,y,z)}_2\,
      \abs{\phi_i(y)\phi_j(z)}\dd y\dd z<\infty.
  \]
  All vector integrals are Bochner integrals.
\item[(A1)] $p$ and $q$ are probability densities in the finite span
  $\operatorname{span}\{\phi_1, \dots, \phi_m\}$.
\item[(A2)] The functions $\phi_1, \dots, \phi_m$ are linearly
  independent.
\item[(A3)] The interaction kernel is antisymmetric:
  $K(x, y^+, y^-) = -K(x, y^-, y^+)$.
\item[(A4)] The matrix $M \in \R^{dN \times r}$ defined in
  \eqref{eq:Mdef} has full column rank, i.e.\ $\rk(M) = r = \binom{m}{2}$.
\item[(A5)] Each coordinate of the stacked empirical drift summand $\Xi_s$
  of \cref{sec:concentration} lies in $[-B_\infty, B_\infty]$ almost surely.
\item[(A6)] $\smin(\widehat M_m)>\varepsilon_M$ (\cref{thm:endtoend}).
\item[(A7)] $\norm{K_X(y^+, y^-)}_2 \le B_X$ for every
  $y^+, y^-\in\R^d$ (\cref{prop:residual}).
\item[(A8)] $p$ and $q$ have finite first moments (\cref{thm:collapse});
  finite third moments for the Gaussian rate (\cref{lem:taurate}) or finite
  second moments for the Laplace rate (\cref{prop:laplace-collapse}).
\item[(A9)] Each pair witness $g_\alpha$ is well defined and Borel measurable
  $\nu$-a.e., and $\norm{G(x)}_{\mathrm{op}} \le L$ for $\nu$-a.e.\ $x$
  (\cref{sec:randprobe}).
\item[(A10)] $\gamma(\nu) = \lmin(\Gamma(\nu)) > 0$
  (\cref{sec:randprobe}).
\end{itemize}
Under (A9)--(A10), \cref{thm:randprobe} supplies (A4) with high probability
for i.i.d.\ probes, so (A4) then becomes a conclusion rather than an
assumption.
For the Gaussian-RBF interaction, \cref{lem:rbf-envelope} verifies (A5) and
(A7) globally with $B_\infty=\sqrt{\tau/e}$ and
$B_X=B_{N,\tau}=\sqrt{N\tau/e}$.  For the Laplace interaction,
\cref{lem:laplace-envelope} gives $B_\infty=\tau/e$ and
$B_X=\sqrt N\,\tau/e$.

\subsection{Exact observability}\label{sec:exact}

\begin{theorem}[Finite-dimensional drift observability]\label{thm:exact}
Assume (A0)--(A4). If $V_X = 0$, then $p = q$ within this
finite-dimensional model class.
\end{theorem}

\begin{proof}
Assume $V_X = 0$. Then $\vecop(V_X) = 0$. By \eqref{eq:obs}, $Mc = 0$.
Since $M$ has full column rank, its nullspace is trivial. Therefore
$c = 0$. By the definition of $c$,
\[
  a_i b_j - a_j b_i = 0 \quad \text{for every } 1 \le i < j \le m.
\]
Equivalently, every $2 \times 2$ minor of the two-column matrix $[a\; b]$
vanishes. Since $q$ is a probability density, $b \ne 0$. Hence there exists
$k$ such that $b_k \ne 0$. The zero-minor relations imply
\[
  a_i = \frac{a_k}{b_k}\, b_i \quad \text{for every } i = 1, \dots, m
\]
(\cref{app:minors}). Thus $a = \lambda b$, with $\lambda = a_k / b_k$.
Therefore $p = \lambda q$. Since both $p$ and $q$ integrate to one,
$\lambda = 1$. Hence $p = q$.
\end{proof}

\subsection{Conditioning-aware stability}\label{sec:conditioning}

For every matrix $A\in\R^{s\times t}$ we use the variational convention
\begin{equation}
  \smin(A):=\inf_{\norm{u}_2=1,\ u\in\R^t}\norm{Au}_2 .
  \label{eq:smin-definition}
\end{equation}
Thus $\smin(A)=0$ whenever $A$ is noninjective, in particular whenever
$s<t$.  This convention prevents the smallest value returned by a compact
SVD of a wide matrix from being mistaken for an injectivity margin.

\begin{proposition}[Deterministic stability]\label{prop:stability}
Assume $M$ has full column rank and let $\smin(M) > 0$ denote its smallest
singular value. Then
\begin{equation}
  \norm{c}_2 \le \frac{\norm{\vecop(V_X)}_2}{\smin(M)}.
  \label{eq:stability}
\end{equation}
\end{proposition}

\begin{proof}
By \eqref{eq:obs}, $\vecop(V_X) = Mc$. Therefore
$\norm{\vecop(V_X)}_2 = \norm{Mc}_2$. Since $M$ has full column rank, the
smallest singular value inequality gives
$\norm{Mc}_2 \ge \smin(M)\, \norm{c}_2$. Combining the two displays gives
$\norm{\vecop(V_X)}_2 \ge \smin(M)\, \norm{c}_2$, and dividing by
$\smin(M) > 0$ gives \eqref{eq:stability}.
\end{proof}

Full rank gives exact identifiability at zero drift. The value of
$\smin(M)$ controls approximate identifiability. If $\smin(M)$ is small,
then small drift can hide a large mismatch vector. Thus the useful
certificate is not raw drift magnitude alone, but raw drift magnitude
divided by an observability scale.

\subsection{Formal versus feasible mismatch directions}\label{sec:feasible}

The vector $c$ in \eqref{eq:cdef} is not an arbitrary element of
$\R^r$. It is generated by two coefficient vectors through the exterior
product relation
\[
  c = a \wedge b, \qquad c_{ij} = a_i b_j - a_j b_i .
\]
Equivalently, $c$ is a decomposable bivector, or the upper triangle of a
skew-symmetric matrix of rank at most two. This structured restriction can
be substantially smaller than the ambient space $\R^r$; the corresponding
Grassmannian dimension count is recorded in \cref{app:grassmannian}.
Consequently, a formal null vector of $M$ need not correspond to any pair
of valid probability densities. Let
\[
  \mathcal C_{\mathrm{feas}}
  := \{ a \wedge b : a,b \in \R^m \text{ make } p,q
      \text{ valid normalized densities in the model class} \}.
\]
The practically relevant observability constant is the restricted value
\[
  \sigma_{\mathrm{feas}}(M)
  := \inf_{\substack{c \in \mathcal C_{\mathrm{feas}}\\ c \ne 0}}
     \frac{\norm{Mc}_2}{\norm{c}_2}.
\]
We use the extended-real convention $\inf\varnothing=+\infty$ if the
declared basis admits no nonzero feasible mismatch.
Since the infimum is taken over a subset of $\R^r$, one has
$\smin(M) \le \sigma_{\mathrm{feas}}(M)$.
Thus the ordinary smallest singular value gives a simple conservative
certificate. If $M$ is rank deficient, there is always a formal blind
direction in $\R^r$; it yields valid $p\ne q$ with zero sampled drift only if
$\ker(M)$ intersects $\mathcal C_{\mathrm{feas}}$ away from zero. This
zero-drift criterion is distinct from pairwise injectivity of
$c\mapsto Mc$ on $\mathcal C_{\mathrm{feas}}$, which would instead require
\[
  \ker(M)\cap
  \bigl(\mathcal C_{\mathrm{feas}}-\mathcal C_{\mathrm{feas}}\bigr)
  =\{0\}.
\]

This distinction is useful in both directions. Full column rank is an easy
sufficient condition that avoids having to characterize
$\mathcal C_{\mathrm{feas}}$. Conversely, when a structural null direction
is found, one must either construct valid coefficients $a,b$ with
$a\wedge b$ in that direction or state the obstruction as a formal
measurement degeneracy. Accordingly, the probe-count results in
\cref{app:randprobes} concern ambient full-column-rank recovery; they are
not claimed to be minimal for distributional identifiability on
$\mathcal C_{\mathrm{feas}}$.

\subsection{From mismatch coordinates to a distributional metric}
\label{sec:bridge}

\Cref{prop:stability} bounds the antisymmetric mismatch vector $c$, not a
distributional metric. The next lemma closes that gap within the model
class: $\norm{c}_2$ controls the coefficient gap $\norm{a - b}_2$, and
hence the $L^1$ (equivalently, total-variation) distance between $p$ and
$q$.

\begin{lemma}[Mismatch-to-coefficient bridge]\label{lem:bridge}
Assume (A1), assume each basis function $\phi_i$ is integrable, and
let $\eta \in \R^m$ have entries
$\eta_i = \int_{\R^d} \phi_i(y) \dd y$. Then
\[
  \norm{a-b}_2\le\norm{\eta}_2\norm{c}_2.
\]
\end{lemma}

\begin{proof}
Let $A := a b^\top - b a^\top \in \R^{m \times m}$. Its $(i,j)$ entry is
$a_i b_j-a_j b_i$, so $A$ is the skew matrix associated with the
decomposable bivector $c=a\wedge b$. Normalization gives
\begin{equation}
  A\eta=a\ip{b}{\eta}-b\ip{a}{\eta}=a-b. \label{eq:Aeta}
\end{equation}
It remains to identify $\norm{A}_{\mathrm{op}}$. Since $b\ne0$, write
$a=\lambda b+w$ with $w\perp b$. Then
$A=wb^\top-bw^\top$. On the orthonormal basis obtained from $w$ and $b$,
the only nonzero block of $A$ is
\[
  \norm{w}_2\norm{b}_2
  \begin{pmatrix}0&1\\-1&0\end{pmatrix}.
\]
(If $w=0$, both sides below are zero.) Hence
\[
  \norm{A}_{\mathrm{op}}
  =\norm{w}_2\norm{b}_2
  =\left(\sum_{i<j}(a_i b_j-a_j b_i)^2\right)^{1/2}
  =\norm{c}_2,
\]
where the middle equality is Lagrange's identity. Applying the operator-norm
inequality to \eqref{eq:Aeta} proves the claim.
\end{proof}

\begin{corollary}[Total-variation certificate within the model class]
\label{cor:tv}
Under (A1), with each $\phi_i\in L^1(\R^d)$, define
\begin{equation}
  \beta_\phi:=\frac12\norm{\eta}_2
  \left(\sum_{i=1}^m\norm{\phi_i}_{L^1}^2\right)^{1/2}. \label{eq:betaphi}
\end{equation}
Then
\[
  \TV(p,q)\le\beta_\phi\norm{c}_2.
\]
If, in addition, each $\phi_i$ is itself a probability density (e.g.\ a
Gaussian-mixture basis), then $\eta_i = 1$ and $\norm{\phi_i}_{L^1} = 1$
for every $i$, so
\[
  \TV(p,q)\le\frac{m}{2}\norm{c}_2.
\]
\end{corollary}

\begin{proof}
By \eqref{eq:finite-basis} and the triangle inequality in $L^1$, then
Cauchy--Schwarz on the finite sum,
\[
  \norm{p - q}_{L^1}
  = \Norm{\sum_{i=1}^m (a_i - b_i) \phi_i}_{L^1}
  \le \sum_{i=1}^m \abs{a_i - b_i}\, \norm{\phi_i}_{L^1}
  \le \norm{a - b}_2 \Big( \sum_{i=1}^m \norm{\phi_i}_{L^1}^2
    \Big)^{1/2}.
\]
Apply \cref{lem:bridge} and use
$\TV(p,q)=\tfrac12\norm{p-q}_{L^1}$ to obtain the first claim. In the
density case, $\norm{\eta}_2=\sqrt m$ and
$(\sum_i\norm{\phi_i}_{L^1}^2)^{1/2}=\sqrt m$, giving
$\norm{p-q}_{L^1}\le m\norm{c}_2$ and hence the stated TV bound.
\end{proof}

\begin{remark}[Bounded IPMs inherit the total-variation certificate]
\label{rem:ipm}
Every integral probability metric (IPM) with a uniformly bounded test-function class inherits the
total-variation certificate; this includes MMDs associated with bounded
kernels. Let $\mathcal F$ satisfy
$\sup_{f\in\mathcal F}\norm{f}_\infty \le B$. Then
\[
  \operatorname{IPM}_{\mathcal F}(p,q)
  := \sup_{f\in\mathcal F}
  \abs{\int f(y)(p(y)-q(y))\dd y}
  \le B\norm{p-q}_{L^1}.
\]
Hence, in the density-basis case of \cref{cor:tv},
\[
  \operatorname{IPM}_{\mathcal F}(p,q)
  \le Bm\norm{c}_2.
\]
For an important example, let $k$ be a positive-semidefinite kernel and let
$\mathcal H_k$ be its reproducing-kernel Hilbert space (RKHS): a Hilbert space
of functions whose kernel sections satisfy
$f(x)=\ip{f}{k(x,\cdot)}_{\mathcal H_k}$. The maximum mean discrepancy is the
IPM over the unit ball of $\mathcal H_k$. If
$\sup_x k(x,x)\le \kappa^2$, then the reproducing property and
Cauchy--Schwarz give $\abs{f(x)}\le \kappa$ whenever
$\norm{f}_{\mathcal H_k}\le 1$, and therefore
\[
  \operatorname{MMD}_k(p,q)
  \le \kappa\norm{p-q}_{L^1}
  \le \kappa m\norm{c}_2.
\]
These inequalities are generic consequences of the TV bridge. They do not
assert that the conditioned drift statistic is a tight or monotone surrogate
for MMD or Wasserstein distance across kernels and bandwidths. IPMs with
unbounded test-function classes, including transport metrics in general,
require additional support or moment assumptions and are not controlled by
\cref{cor:tv} alone. Likewise, TV alone does not upper-bound either direction
of KL divergence without appropriate absolute-continuity and density-ratio
control.
\end{remark}

\begin{remark}[Sharpness and the price of resolution]
\label{rem:bridgesharp}
The coefficient bridge can be attained exactly, so it is not merely a loose
generic estimate. For $m=2$ with density basis, let
$a=(1,0)$, $b=(1-t,t)$, and $t\in(0,1)$. Then
\[
  \norm{c}_2=t,
  \qquad
  \norm{a-b}_2=t\sqrt{2},
\]
which equals the bound $\sqrt{2}\,t$ from \cref{lem:bridge}. Thus the
coefficient bridge is attained in this example.

Increasing the basis size permits finer representational resolution, but the
generic worst-case TV conversion in the displayed bridge scales linearly in
$m$. Whether this dependence is optimal for a narrower basis family requires
separate analysis. All certificates in \cref{sec:certification} can be read in
total variation by composing them with \cref{cor:tv}.
\end{remark}

\section{Finite-Observation Certification}\label{sec:certification}

The deterministic inequality in \cref{prop:stability} gives the starting
point: observed drift controls mismatch only through an observability margin.
Allowing for observation and representation errors leads to the general
inverse inequality below. Calibrated drift and operator errors, together with
random-probe conditions for a positive margin, then yield the held-out
total-variation audit.

\subsection{Approximate observation model}\label{sec:approximation}

Let $p$ and $q$ be densities that need not lie exactly in
$\operatorname{span}\{\phi_1,\dots,\phi_m\}$. Write
\begin{equation}
  p=p_m+r_p,\qquad q=q_m+r_q, \label{eq:approx-split}
\end{equation}
where
\[
  p_m(y)=\sum_{i=1}^m a_i\phi_i(y),\qquad
  q_m(y)=\sum_{i=1}^m b_i\phi_i(y),
\]
and $r_p,r_q$ are residual terms. For the finite parts $(p_m,q_m)$,
define $M_m$ and $c_m$ exactly as in \cref{sec:observability}. The full
sampled drift decomposes as
\begin{equation}
  \vecop(V_X)=M_m c_m+R_m, \label{eq:approx-decomp}
\end{equation}
where $R_m\in\R^{dN}$ is the drift contribution of the representation
residuals. \Cref{prop:residual} bounds $R_m$ under a global stacked-kernel
envelope.  The drift observations do not estimate
$\norm{r_p}_{L^1}$ or $\norm{r_q}_{L^1}$: any radii for these quantities are
external approximation certificates. For the full-distribution TV
consequences, $p_m$ and $q_m$ must themselves be normalized probability
densities. Without such radii, the formulas below provide only a conditional
sensitivity curve for the chosen finite model.

For a held-out drift estimate, draw independent pairs
$(Y_s^+,Y_s^-)\overset{\mathrm{iid}}{\sim}p\otimes q$ and set
\begin{equation}
  \Vhat_X[:,\ell]
  =\frac1n\sum_{s=1}^n K(x_\ell,Y_s^+,Y_s^-),
  \qquad \ell=1,\dots,N. \label{eq:empdrift}
\end{equation}
An observation matrix may likewise be computed or approximated; denote it by
$\widehat M_m$.

\subsection{Master inverse certificate}\label{sec:master-certificate}

\begin{theorem}[End-to-end finite-observation certificate]
\label{thm:endtoend}
Assume the approximate observation equation \eqref{eq:approx-decomp}, let
$r=\binom m2$, and suppose $dN\ge r$ and $\widehat M_m\in\R^{dN\times r}$
has full column rank.  Suppose
\begin{equation}
  \Norm{\vecop(\Vhat_X)-\vecop(V_X)}_2\le\varepsilon_V,
  \qquad
  \Norm{\widehat M_m-M_m}_{\mathrm{op}}\le\varepsilon_M,
  \label{eq:master-errors}
\end{equation}
and assume (A6), namely $\smin(\widehat M_m)>\varepsilon_M$. Then
\begin{equation}
  \norm{c_m}_2
  \le
  \frac{\Norm{\vecop(\Vhat_X)}_2+\varepsilon_V+\norm{R_m}_2}
       {\smin(\widehat M_m)-\varepsilon_M}.
  \label{eq:endtoend}
\end{equation}
\end{theorem}

\begin{proof}
The decomposition and two triangle inequalities give
\[
  \norm{M_m c_m}_2
  \le \norm{\vecop(V_X)}_2+\norm{R_m}_2
  \le \Norm{\vecop(\Vhat_X)}_2+\varepsilon_V+\norm{R_m}_2.
\]
Weyl's singular-value perturbation inequality gives
$\smin(M_m)\ge\smin(\widehat M_m)-\varepsilon_M>0$. Hence
\[
  (\smin(\widehat M_m)-\varepsilon_M)\norm{c_m}_2
  \le\smin(M_m)\norm{c_m}_2
  \le\norm{M_m c_m}_2,
\]
and combining the displays proves \eqref{eq:endtoend}.
\end{proof}

\begin{remark}[Exact specializations]\label{rem:master-specializations}
The preceding population, empirical, exact-basis, and exact-operator settings
are recovered as special cases by setting the corresponding errors to zero.
If $p=p_m$ and $q=q_m$, then $R_m=0$, $c_m=c$, and $M_m=M$. If the
observation operator is known exactly, take $\widehat M_m=M_m$ and
$\varepsilon_M=0$. If the population drift is observed exactly, take
$\Vhat_X=V_X$ and $\varepsilon_V=0$. Combining these identifications yields
the empirical exact-basis, perturbed-operator, and approximate-basis bounds;
imposing all three exactness conditions recovers \cref{prop:stability}.
\end{remark}

The numerator in \eqref{eq:endtoend} consists of the observed held-out drift,
the statistical error radius, and the representation residual. The
denominator is a lower bound on the observability scale after accounting for
operator-estimation error. If this denominator is nonpositive, the inverse
bound is inconclusive.  Implementations must use the variational convention
\eqref{eq:smin-definition}. For $dN<r$,
$\smin(\widehat M_m)=0$, so the error-adjusted observability margin is
nonpositive and the audit must abstain. In particular, one must not use the
smallest of the $\min\{dN,r\}$ values returned by a compact SVD of a wide
matrix.

\subsection{Held-out drift error bounds}\label{sec:concentration}

For the held-out drift estimator, take $D=dN$ and
\begin{align*}
  \Xi_s&:=\vecop\big(K(x_1,Y_s^+,Y_s^-)\;\cdots\;
                    K(x_N,Y_s^+,Y_s^-)\big)\in\R^D,\\
  \hat z&:=\frac1n\sum_{s=1}^n \Xi_s,
  \qquad z:=\E \Xi_s=\vecop(V_X).
\end{align*}
The next lemma is stated for generic i.i.d. vectors; part (b) will also be
used for operator calibration.

\begin{lemma}[Held-out drift error radii]\label{lem:drift-radii}
Let $\Xi_1,\dots,\Xi_n$ be i.i.d. random vectors in $\R^D$, write
$\hat z=n^{-1}\sum_{s=1}^n \Xi_s$ and $z=\E \Xi_s$, and fix
$\delta\in(0,1)$.
\begin{enumerate}
\item[(a)] \emph{Coordinate Hoeffding.}
If every coordinate of $\Xi_s$ lies in $[-B_\infty,B_\infty]$ almost surely,
then, with probability at least $1-\delta$,
\begin{equation}
  \norm{\hat z-z}_2
  \le B_\infty\sqrt{\frac{2D\log(2D/\delta)}{n}}.
  \label{eq:hoeffding-bound}
\end{equation}

\item[(b)] \emph{Bounded vector.}
If $\norm{\Xi_s}_2\le B_2$ almost surely, then, with probability at least
$1-\delta$,
\begin{equation}
  \norm{\hat z-z}_2
  \le \frac{B_2}{\sqrt n}
  \left(1+\sqrt{2\log\frac1\delta}\right).
  \label{eq:vector-mean-bound}
\end{equation}
For the Gaussian-RBF interaction, \cref{lem:rbf-envelope} gives
$B_2=B_{N,\tau}=\sqrt{N\tau/e}$ and hence the explicit radius
\begin{equation}
  \varepsilon_V^{\mathrm{vec}}(\delta)
  :=\sqrt{\frac{N\tau}{en}}
    \left(1+\sqrt{2\log\frac1\delta}\right).
  \label{eq:epsV-vector}
\end{equation}

\item[(c)] \emph{Coordinate empirical Bernstein.}
Assume the coordinate envelope in part (a), $n\ge2$, and define
\[
  \hat v_\rho:=\frac1{n-1}\sum_{s=1}^n
    \big((\Xi_s)_\rho-\hat z_\rho\big)^2,
  \qquad \rho=1,\dots,D.
\]
Set
\begin{equation}
  r_\rho(\delta)
  :=\sqrt{\frac{2\hat v_\rho\log(4D/\delta)}{n}}
    +\frac{14B_\infty\log(4D/\delta)}{3(n-1)},
  \qquad
  \varepsilon_V^{\mathrm{EB}}(\delta)
  :=\left(\sum_{\rho=1}^D r_\rho(\delta)^2\right)^{1/2}.
  \label{eq:epsV-EB}
\end{equation}
Then, with probability at least $1-\delta$,
$\norm{\hat z-z}_2\le\varepsilon_V^{\mathrm{EB}}(\delta)$.
\end{enumerate}
\end{lemma}

The three proofs are collected in \cref{app:radius-proofs}. Part (a) is the
most elementary; part (b) avoids an explicit coordinate union bound; part
(c) adapts to the observed coordinate variances.
Each displayed radius has level $1-\delta$ when used on its own.  The audit
must prespecify the radius rule.  If several radii are computed and the
smallest is selected after observing the batch, their failure probabilities
must instead be made simultaneous, for example by evaluating three methods
at level $\delta/3$ and taking their minimum.

\begin{corollary}[High-probability exact-basis certificate]
\label{cor:hpcert}
Assume the exact observation equation $\vecop(V_X)=Mc$, that $M$ has full
column rank, and the coordinate envelope in
\cref{lem:drift-radii}(a). Then, with probability at least $1-\delta$,
\begin{equation}
  \norm{c}_2
  \le
  \frac{\Norm{\vecop(\Vhat_X)}_2
    +B_\infty\sqrt{\dfrac{2dN\log(2dN/\delta)}{n}}}
       {\smin(M)}.
  \label{eq:hp-cert}
\end{equation}
The Hoeffding term may instead be replaced by the right-hand side of
\eqref{eq:vector-mean-bound}; for the Gaussian-RBF interaction this is
\eqref{eq:epsV-vector}. Under the assumptions of
\cref{lem:drift-radii}(c), it may instead be replaced by
$\varepsilon_V^{\mathrm{EB}}(\delta)$ from \eqref{eq:epsV-EB}.
\end{corollary}

\begin{proof}
Apply \cref{thm:endtoend} with
$p=p_m$, $q=q_m$, $R_m=0$, $c_m=c$, $M_m=M$,
$\widehat M_m=M$, and $\varepsilon_M=0$, then insert the selected part of
\cref{lem:drift-radii} for $\varepsilon_V$.
\end{proof}

\subsection{Observation-operator calibration}\label{sec:estM}

When the basis integrals defining the observation matrix are evaluated by
quadrature or Monte Carlo, the matrix error in \eqref{eq:master-errors} must
be calibrated rather than ignored.

\begin{proposition}[Monte Carlo calibration of the observation matrix]
\label{prop:M-calibration}
Assume each $\phi_i$ is a samplable probability density. For each
$\alpha=(i,j)\in P$, let $s_\alpha\ge1$ and draw
$(A_{\alpha t},B_{\alpha t})\overset{\mathrm{iid}}{\sim}
\phi_i\otimes\phi_j$, $t=1,\dots,s_\alpha$. Define
\[
  W_{\alpha t}
  :=\vecop\big(K(x_1,A_{\alpha t},B_{\alpha t})\;\cdots\;
               K(x_N,A_{\alpha t},B_{\alpha t})\big)\in\R^{dN}
\]
and let column $\alpha$ of $\widehat M_m$ be
$\hat u_\alpha=s_\alpha^{-1}\sum_t W_{\alpha t}$. Suppose
$\norm{W_{\alpha t}}_2\le B_\alpha$ almost surely. For every
$\delta_M\in(0,1)$, with probability at least $1-\delta_M$,
\begin{equation}
  \norm{\widehat M_m-M_m}_{\mathrm{op}}
  \le\varepsilon_M^{\mathrm{MC}}
  :=\left[\sum_{\alpha\in P}\frac{B_\alpha^2}{s_\alpha}
    \left(1+\sqrt{2\log\frac{r}{\delta_M}}\right)^2\right]^{1/2}.
  \label{eq:epsM-MC}
\end{equation}
For the Gaussian-RBF interaction one may take
$B_\alpha=B_{N,\tau}=\sqrt{N\tau/e}$ for every $\alpha$.
\end{proposition}

The proof is deferred to \cref{app:proof-M-calibration}. In numerical work,
$\smin(\widehat M_m)$ in \cref{thm:endtoend} must be replaced by a
certified lower bound $\underline\sigma\le\smin(\widehat M_m)$.  One
software-independent construction follows.

\begin{proposition}[Outward-rounded Gram lower bound]
\label{prop:gram-lower}
Let $A\in\R^{D\times r}$ with $D\ge r$.  Suppose outward-rounded
arithmetic supplies intervals
$[g_{ij}^-,g_{ij}^+]$ containing every entry of $A^\top A$.  Define
\[
  \underline\lambda
  :=\min_{1\le i\le r}\left\{
    g_{ii}^--\sum_{j\ne i}
      \max\{\abs{g_{ij}^-},\abs{g_{ij}^+}\}\right\},
  \qquad
  \underline\sigma_{\mathrm{Gram}}
  :=\sqrt{\max\{0,\underline\lambda\}} .
\]
Then $\underline\sigma_{\mathrm{Gram}}\le\smin(A)$.  For $r=1$, this
reduces to an outward-rounded lower bound on the column norm (while
$\underline\lambda$ bounds its squared norm).
\end{proposition}

\begin{proof}
Every exact Gram entry lies in its interval.  Gershgorin's theorem therefore
gives $\lambda_{\min}(A^\top A)\ge\underline\lambda$.  The claim follows
from $\smin(A)^2=\lambda_{\min}(A^\top A)$ when $A$ has full column rank,
and is trivial after the maximum with zero otherwise.
\end{proof}

This bound is conservative but directly auditable; a verified interval
eigensolver may replace it.  If the entries of the stored matrix enclose an
ideal Monte Carlo or quadrature matrix only up to an operator-norm rounding
radius $\eta_{\mathrm{fp}}$, add $\eta_{\mathrm{fp}}$ to
$\varepsilon_M$ before forming the margin.  Deterministic quadrature and
storage errors are handled in the same way.  A nonpositive margin
$\underline\sigma-\varepsilon_M$ triggers abstention.

\subsection{Randomized probe selection}\label{sec:randprobe}

The inverse certificate takes the probe set, and hence $\smin(M)$, as given.
This subsection shows that i.i.d.\ random probes yield a full-rank,
quantitatively well-conditioned observation matrix with high probability,
with the conditioning governed by a single population quantity.

For a pair $\alpha = (i,j) \in P$, define the \emph{pair witness function}
\begin{equation}
  g_\alpha(x)
  := \int\!\!\int K(x, y^+, y^-)\, \phi_i(y^+) \phi_j(y^-)
     \dd y^+ \dd y^- \;\in\; \R^d, \label{eq:pairwitness}
\end{equation}
so that $U_{ij}[:, \ell] = g_{ij}(x_\ell)$ by \eqref{eq:Uij}. With $P$
ordered as in \eqref{eq:Mdef}, collect these functions into the
\emph{probe block}
\begin{equation}
  G(x) := [\, g_{\alpha_1}(x) \;\cdots\; g_{\alpha_r}(x) \,]
  \in \R^{d \times r}. \label{eq:probeblock}
\end{equation}
Under column-wise vectorization, column $s$ of $M$ is
\begin{equation}
  \vecop(U_{\alpha_s})
  = \big(g_{\alpha_s}(x_1);\, \dots;\, g_{\alpha_s}(x_N)\big).
  \label{eq:Mcolumn}
\end{equation}
Thus $M$ consists of the stacked blocks $G(x_1), \dots, G(x_N)$, and
\begin{equation}
  M^\top M = \sum_{\ell=1}^{N} G(x_\ell)^\top G(x_\ell).
  \label{eq:gramdecomp}
\end{equation}
(Any other consistent stacking permutes the rows of $M$ and leaves all
singular values unchanged, so \eqref{eq:gramdecomp} is
convention-independent.)

Let $\nu$ be a probability measure on $\R^d$ and
$X_1, \dots, X_N \overset{\mathrm{iid}}{\sim} \nu$, independent of the
minibatch in \eqref{eq:empdrift}. Define the \emph{population observability
Gram matrix}
\begin{equation}
  \Gamma(\nu) := \E\big[ G(X_1)^\top G(X_1) \big] \in \R^{r \times r},
  \qquad
  \Gamma(\nu)_{\alpha \beta}
  = \int \ip{g_\alpha(x)}{g_\beta(x)} \dd \nu(x), \label{eq:Gamma}
\end{equation}
i.e.\ the Gram matrix of $\{g_\alpha\}_{\alpha \in P}$ in
$L^2(\nu; \R^d)$, and set $\gamma(\nu) := \lmin(\Gamma(\nu))$.
Thus $\gamma(\nu)$ is the squared population observability margin, while
$\sqrt{\gamma(\nu)}$ is the singular-value scale. We use
assumptions (A9), $\norm{G(x)}_{\mathrm{op}} \le L$ for $\nu$-a.e.\ $x$
(a sufficient check is
$\sup_x (\sum_\alpha \norm{g_\alpha(x)}_2^2)^{1/2} \le L$, since the
operator norm is bounded by the Frobenius norm), and (A10),
$\gamma(\nu) > 0$.

\begin{lemma}[Structure and necessity of (A10)]\label{lem:gram}
Assume $g_\alpha\in L^2(\nu;\R^d)$ for every $\alpha\in P$.
\emph{(a)} $\gamma(\nu) > 0$ if and only if
$\{g_\alpha\}_{\alpha \in P}$ are linearly independent in
$L^2(\nu; \R^d)$, i.e.\
$\sum_\alpha \lambda_\alpha g_\alpha(x) = 0$ for $\nu$-a.e.\ $x$ implies
$\lambda = 0$.
\emph{(b)} If $\gamma(\nu) = 0$, then $\rk(M) < r$ almost surely, for
every $N$.
\end{lemma}

\begin{proof}
(a) For $\lambda \in \R^r$ write
$g_\lambda := \sum_\alpha \lambda_\alpha g_\alpha$. By \eqref{eq:Gamma}
and linearity of the integral over the finite sum,
\[
  \lambda^\top \Gamma(\nu)\, \lambda
  = \sum_{\alpha, \beta} \lambda_\alpha \lambda_\beta
    \int \ip{g_\alpha(x)}{g_\beta(x)} \dd \nu(x)
  = \int \norm{g_\lambda(x)}_2^2 \dd \nu(x)
  = \norm{g_\lambda}_{L^2(\nu)}^2.
\]
Hence $\Gamma(\nu) \succ 0$ iff no $\lambda \ne 0$ has $g_\lambda = 0$
$\nu$-a.e.

(b) If $\gamma(\nu) = 0$, part (a) supplies $\lambda \ne 0$ with
$\nu(Z) = 1$ for $Z := \{x : g_\lambda(x) = 0\}$. By independence,
$\PP(X_\ell \in Z \text{ for all } \ell) = 1$, and on that event
$M \lambda = (g_\lambda(X_1); \dots; g_\lambda(X_N)) = 0$, so the
nullspace of $M$ is nontrivial.
\end{proof}

\Cref{lem:gram}(b) shows that (A10) is necessary for recovery over the full
ambient mismatch space from i.i.d.\ probes; together with (A9) and
sufficiently large $N$, it yields the quantitative high-probability
conditioning below.
\Cref{ex:midpoint} exhibits a natural
basis for which $\gamma(\nu) = 0$ for every $\nu$; actual distributional
non-identifiability additionally requires feasible coefficients as in
\cref{sec:feasible}.

\begin{theorem}[High-probability well-conditioning under random probes]
\label{thm:randprobe}
Assume (A9)--(A10) and let $M$ be built from
$X_1, \dots, X_N \overset{\mathrm{iid}}{\sim} \nu$ as in \eqref{eq:Mdef}.
Then for every $\varepsilon \in (0, 1)$,
\[
  \PP\Big( \lmin(M^\top M) \le (1 - \varepsilon)\, N \gamma(\nu) \Big)
  \le r \exp\Big( -\frac{\varepsilon^2 N \gamma(\nu)}{2 L^2} \Big).
\]
In particular, for any $\delta' \in (0, 1)$, if
$N \ge \dfrac{8 L^2}{\gamma(\nu)} \log \dfrac{r}{\delta'}$, then with
probability at least $1 - \delta'$ the matrix $M$ has full column rank
(so (A4) holds) and
\[
  \smin(M) \ge \sqrt{\tfrac{1}{2} N \gamma(\nu)}.
\]
\end{theorem}

The proof is deferred to \cref{app:proof-randprobe}.

\begin{corollary}[Random-probe empirical certificate]
\label{cor:randcert}
Assume (A0)--(A3), (A5), and (A9)--(A10). Let
$X_1,\dots,X_N\overset{\mathrm{iid}}{\sim}\nu$ be independent of the
minibatch pairs in \eqref{eq:empdrift}, and suppose the bound in (A5) holds
almost surely jointly in probes and minibatch. If
\[
  N \ge \frac{8L^2}{\gamma(\nu)}
    \log\!\left(\frac{r}{\delta'}\right),
\]
then, with probability at least $1-\delta-\delta'$,
\[
  \norm{c}_2
  \le \frac{\Norm{\vecop(\Vhat_X)}_2
    + B_\infty \sqrt{\dfrac{2 d N \log(2 d N / \delta)}{n}}}
    {\sqrt{N \gamma(\nu) / 2}}.
\]
\end{corollary}

\begin{proof}
Let $E_1$ be the event of \cref{thm:randprobe}, with
$\PP(E_1^c)\le\delta'$. Conditionally on any probe realization,
\cref{lem:drift-radii}(a) gives a drift-error event $E_2$ satisfying
$\PP(E_2^c\mid X_1,\dots,X_N)\le\delta$, hence
$\PP(E_2^c)\le\delta$. On $E_1\cap E_2$, the exact-basis,
exact-operator specialization of \cref{thm:endtoend} applies with
$\smin(M)\ge\sqrt{N\gamma(\nu)/2}$. A union bound gives the stated
probability and substituting the lower bound gives the display.
\end{proof}

\begin{corollary}[Total-variation form of the empirical certificate]
\label{cor:tvchain}
Assume (A0)--(A5) with each $\phi_i$ a probability density. Then with
probability at least $1 - \delta$,
\[
  \TV(p, q)
  \le \frac{m}{2}\,
  \frac{\Norm{\vecop(\Vhat_X)}_2
    + B_\infty \sqrt{\dfrac{2 d N \log(2 d N / \delta)}{n}}}
    {\smin(M)}.
\]
Under the hypotheses of \cref{cor:randcert}, the denominator may be
replaced by $\sqrt{N \gamma(\nu) / 2}$ at the cost of decreasing the
probability to $1 - \delta - \delta'$.
\end{corollary}

\begin{proof}
Compose \cref{cor:hpcert} with \cref{cor:tv} (density case); the second
statement composes \cref{cor:randcert} with \cref{cor:tv}.
\end{proof}

\begin{remark}[Scaling in the number of probes]\label{rem:scaling}
Both sides of the certificate scale consistently in $N$: the numerator
aggregates $N$ probe evaluations and generically grows like $\sqrt{N}$,
while the denominator grows like $\sqrt{N \gamma(\nu) / 2}$. Hence the
radius-free ratio $\norm{\vecop(\Vhat_X)}_2/\smin(M)$ typically stabilizes
as $N$ grows rather than degenerating, and $\sqrt{\gamma(\nu)}$ emerges as
the population-level singular-value scale.  In a certificate, the denominator is
estimated either by $\smin(\widehat M_m)$ (with the error margin of
\cref{thm:endtoend}) or, for the Gaussian benchmark model of
\cref{app:randprobes}, estimated from the closed-form witnesses. Note also that
$\gamma(\nu)$ can be positive but small (e.g.\ near-colliding pair
witnesses; cf.\ \cref{ex:midpoint}), in which case the probe requirement
$N \gtrsim L^2 / \gamma(\nu)$ becomes large. This is the random-probe
analogue of the conditioning phenomenon in \cref{prop:stability}.
\end{remark}

\subsection{Distributional and held-out consequences}\label{sec:distributional-consequences}

\begin{corollary}[Full-distribution total-variation certificate]
\label{cor:fulltv}
In addition to the assumptions of \cref{thm:endtoend}, suppose each
$\phi_i$ is a probability density and the finite parts $p_m$ and $q_m$ in
\eqref{eq:approx-split} are themselves normalized probability densities.
If $r_p,r_q\in L^1$, then
\begin{equation}
  \TV(p,q)
  \le \frac{1}{2}\norm{r_p}_{L^1}
    +\frac{m}{2}\,
      \frac{\Norm{\vecop(\Vhat_X)}_2+\varepsilon_V+\norm{R_m}_2}
           {\smin(\widehat M_m)-\varepsilon_M}
    +\frac{1}{2}\norm{r_q}_{L^1}.
  \label{eq:fulltv}
\end{equation}
\end{corollary}

\begin{proof}
By the triangle inequality in $L^1$ and \cref{cor:tv} applied to the
normalized finite parts,
\[
  \TV(p,q)
  =\tfrac12\norm{p-q}_{L^1}
  \le \tfrac12\norm{r_p}_{L^1}
      +\TV(p_m,q_m)
      +\tfrac12\norm{r_q}_{L^1}
  \le \tfrac12\norm{r_p}_{L^1}
      +\frac{m}{2}\norm{c_m}_2
      +\tfrac12\norm{r_q}_{L^1}.
\]
Apply \cref{thm:endtoend}.
\end{proof}

\subsection{Held-out application}\label{sec:heldout}

For a held-out use of the preceding bounds, let $\mathcal F_0$ be the
$\sigma$-field generated by the generator, basis, kernel, bandwidth, and
probes, all frozen before fresh audit pairs are drawn i.i.d.\ from
$p\otimes q$. Let $\delta_V,\delta_M,\delta_R\in[0,1]$ satisfy
$\delta_V+\delta_M+\delta_R\le1$. Let
$\varepsilon_V,\varepsilon_M,\rho_p,\rho_q$ be valid radii for which the
following conditional-probability statements hold almost surely: the drift
event has probability at least $1-\delta_V$; an operator-calibration event
provides a matrix $\widetilde M$ and certified lower bound
$\underline\sigma\le\smin(\widetilde M)$ with probability at least
$1-\delta_M$; and
\begin{equation}
  \norm{r_p}_{L^1}\le\rho_p,
  \qquad
  \norm{r_q}_{L^1}\le\rho_q
  \label{eq:residual-envelope}
\end{equation}
holds with probability at least $1-\delta_R$.  These are external
inputs to the audit, not quantities inferred from the drift batch.  Here the
drift and operator events
are respectively
$\norm{\vecop(\Vhat_X)-\vecop(V_X)}_2\le\varepsilon_V$ and
$\norm{\widetilde M-M_m}_{\mathrm{op}}\le\varepsilon_M$; the three events
need not be mutually independent.  Assume the bounded-kernel conditions of
\cref{prop:residual}, and set
\begin{align}
  \overline R_m&:=B_X(\rho_p+\rho_q+\rho_p\rho_q),
  \label{eq:Rbar}\\
  \Delta&:=\underline\sigma-\varepsilon_M.
  \label{eq:Delta-audit}
\end{align}
When $\Delta>0$, define
\begin{equation}
  U_c:=\frac{\norm{\vecop(\Vhat_X)}_2+\varepsilon_V+\overline R_m}{\Delta},
  \qquad
  U_{\mathrm{TV}}
  :=\min\left\{1,\frac{\rho_p+\rho_q}{2}+\beta_\phi U_c\right\}.
  \label{eq:operational-bounds}
\end{equation}
If $\Delta\le0$, set $U_c=+\infty$ and $U_{\mathrm{TV}}=1$.

\begin{corollary}[Held-out total-variation certificate]
\label{cor:operational}
Assume \eqref{eq:approx-decomp}, the conditions above, and that the finite
parts $p_m,q_m$ are normalized probability densities sharing the integrable
basis required by \cref{cor:tv}.  Then
\begin{equation}
  \PP\{\TV(p,q)\le U_{\mathrm{TV}}\mid\mathcal F_0\}
  \ge1-\delta_V-\delta_M-\delta_R
  \quad\text{almost surely}.
  \label{eq:operational-coverage}
\end{equation}
Taking expectations yields the same unconditional coverage bound.
Thus $\Delta\le0$ returns the universally valid bound $1$, interpreted as an
inconclusive audit rather than evidence of distribution matching.
\end{corollary}

\begin{proof}
On the intersection of the three radius events, Weyl's inequality gives
$\smin(M_m)\ge\smin(\widetilde M)-\varepsilon_M\ge\Delta$.
If $\Delta>0$, \cref{thm:endtoend,prop:residual} give
$\norm{c_m}_2\le U_c$, while the $L^1$ triangle inequality and
\cref{cor:tv} give
\[
  \TV(p,q)
  \le\tfrac12\norm{r_p}_{L^1}+\beta_\phi\norm{c_m}_2
      +\tfrac12\norm{r_q}_{L^1}
  \le\tfrac{\rho_p+\rho_q}{2}+\beta_\phi U_c.
\]
Capping at $1$ preserves validity; when $\Delta\le0$, the reported bound is
already $1$.  A conditional union bound given $\mathcal F_0$ proves
\eqref{eq:operational-coverage}.
\end{proof}

\begin{corollary}[Total-variation equivalence test]
\label{cor:operational-equivalence}
For a prespecified $\epsilon>0$, test $H_0:\TV(p,q)\ge\epsilon$ against
$H_1:\TV(p,q)<\epsilon$ and reject only when
$U_{\mathrm{TV}}<\epsilon$. Under the assumptions of
\cref{cor:operational}, the type-I error is at most
$\delta_V+\delta_M+\delta_R$.
\end{corollary}

\begin{proof}
Under $H_0$, rejection implies $U_{\mathrm{TV}}<\TV(p,q)$ and is contained
in the complement of the coverage event in
\eqref{eq:operational-coverage}.
\end{proof}

In the exact finite-basis setting, identify $p=p_m$ and $q=q_m$ and take
$\rho_p=\rho_q=\overline R_m=0$ and $\delta_R=0$. If the operator is
available exactly, identify $\widetilde M=M_m$ and take
$\varepsilon_M=\delta_M=0$. If $\zeta_p,\zeta_q\in[0,1]$ and
$p_m,q_m,h_p,h_q$ are probability densities satisfying
$p=(1-\zeta_p)p_m+\zeta_p h_p$ and
$q=(1-\zeta_q)q_m+\zeta_q h_q$, use
$\rho_p=2\zeta_p$ and $\rho_q=2\zeta_q$, since
$\norm{h_p-p_m}_{L^1},\norm{h_q-q_m}_{L^1}\le2$. Without a justified
residual envelope, report the conditional sensitivity curve
$(\rho_p,\rho_q)\mapsto U_{\mathrm{TV}}$ rather than an unqualified
full-distribution certificate. Reusing the audit batch to select the
generator, probes, bandwidth, or basis requires a separate uniform or
sequential correction.

The end-to-end certificate has the form
\[
  \text{mismatch}
  \;\lesssim\;
  \frac{\text{drift}+\text{sampling error}+\text{basis residual}}
    {\text{estimated observability}-\text{matrix error}}
  +\text{density residual}.
\]
The main extension beyond exact finite-basis identifiability is the explicit
treatment of approximation and estimation error.
Equation~\eqref{eq:approx-decomp} identifies the field residual induced by model
approximation; \eqref{eq:endtoend} propagates drift, operator, and
representation errors through the inverse bound; and \eqref{eq:fulltv} adds
the $L^1$ density residuals. \Cref{cor:operational} assigns valid radii to
these terms, combines the corresponding events by a union bound, and returns
the trivial TV bound when the error-adjusted observability margin is
nonpositive.

\section{Large-Bandwidth Collapse}\label{sec:collapse}

A similarity kernel that is too flat creates a basic failure mode for
drift-based training: very small drift can coexist with a large
distributional mismatch.

Consider the unnormalized mean-shift-style drift with the Gaussian-RBF
kernel \eqref{eq:rbf}. The corresponding drift is
\begin{equation}
  V^{(\tau)}_{p,q}(x)
  = \E_{Y^+ \sim p,\, Y^- \sim q}
    \big[ k_\tau(x, Y^+)\, k_\tau(x, Y^-)\, (Y^+ - Y^-) \big].
  \label{eq:tau-drift}
\end{equation}

\begin{theorem}[Large-bandwidth collapse to mean matching]
\label{thm:collapse}
Assume $p$ and $q$ have finite first moments. Then, for every fixed
$x \in \R^d$,
\begin{equation}
  V^{(\tau)}_{p,q}(x)
  \longrightarrow \E_{Y \sim p}[Y] - \E_{Y \sim q}[Y]
  \quad \text{as } \tau \to \infty. \label{eq:collapse}
\end{equation}
\end{theorem}

The proof is deferred to \cref{app:collapse-proofs}.

\begin{corollary}[The normalized population field has the same limit]
\label{cor:normalized-collapse}
Under the assumptions of \cref{thm:collapse}, let
$\bar V^{(\tau)}_{p,q}(x)=V^{(\tau)}_{p,q}(x)/Z_\tau(x)$ use the
normalization of \cref{rem:normalized}, with $k=k_\tau$. Then, for every
fixed $x\in\R^d$,
\[
  \bar V^{(\tau)}_{p,q}(x)
  \longrightarrow \E_p[Y]-\E_q[Y].
\]
\end{corollary}

The proof is included in \cref{app:collapse-proofs}.

\begin{corollary}[Mean-matched distributions can become
drift-indistinguishable]\label{cor:meanmatched}
Let $p \ne q$ be two distributions with finite first moments and equal
means: $\E_{Y \sim p}[Y] = \E_{Y \sim q}[Y]$. Then, for every fixed
$x \in \R^d$,
\[
  V^{(\tau)}_{p,q}(x) \to 0 \quad \text{as } \tau \to \infty,
\]
even though $p \ne q$; the normalized population field of
\cref{cor:normalized-collapse} also converges to zero.
\end{corollary}

\begin{proof}
Because the means are equal, the limit in \cref{thm:collapse} is zero.
The normalized conclusion follows from \cref{cor:normalized-collapse}.
\end{proof}

The next lemma makes the collapse quantitative, under third moments.
Throughout, $\Delta\mu := \E_{Y \sim p}[Y] - \E_{Y \sim q}[Y]$.

\begin{lemma}[Quantitative large-bandwidth collapse]\label{lem:taurate}
Let $k_\tau$ be the Gaussian-RBF kernel \eqref{eq:rbf}, let
$Y^+ \sim p$ and $Y^- \sim q$ be independent with
$\E \norm{Y^+}_2^3 < \infty$ and $\E \norm{Y^-}_2^3 < \infty$. Then for
every fixed $x \in \R^d$,
\[
  \Norm{V^{(\tau)}_{p,q}(x) - \Delta\mu}_2
  \le \frac{C(x, p, q)}{\tau},
  \qquad
  C(x, p, q)
  := \E\Big[ \big( \norm{x - Y^+}_2^2 + \norm{x - Y^-}_2^2 \big)\,
    \norm{Y^+ - Y^-}_2 \Big] < \infty.
\]
\end{lemma}

Its proof is included in \cref{app:collapse-proofs}.

\begin{remark}[Consequences]\label{rem:taurate}
(i) \Cref{thm:collapse} follows from \cref{lem:taurate} by letting
$\tau \to \infty$, under the (stronger) third-moment hypothesis.
(ii) Since $C(x, p, q)$ grows at most quadratically in $\norm{x}_2$,
summing over probes gives the probe-set bound
$\norm{V^{(\tau)}_X - \Delta\mu\, \mathbf{1}_N^\top}_F
\le \tau^{-1} \sum_{\ell=1}^N C(x_\ell, p, q)$.
(iii) For mean-matched pairs ($\Delta\mu = 0$), the lemma guarantees the
upper rate $\norm{V_X^{(\tau)}}_F=O(1/\tau)$ on every fixed finite probe
set. It does not supply a matching lower bound: cancellation of the
first-order term can produce faster decay. \Cref{sec:exp-bandwidth} reports
numerical decay with slope near $-1$ for the chosen benchmark, consistent
with this upper rate,
while $W_2$ and MMD computed with a separate, fixed evaluation kernel remain
positive and unchanged by the audit bandwidth $\tau$.  An MMD kernel whose
own bandwidth equals $\tau$ would instead generally change with $\tau$ and
is not meant here.
\end{remark}

\begin{proposition}[Laplace-kernel analogue]\label{prop:laplace-collapse}
Let $V^{(\tau),\mathrm{Lap}}_{p,q}$ denote \eqref{eq:tau-drift} with
$k_\tau$ replaced by $\ell_\tau$ from \eqref{eq:laplace}.  If $p$ and $q$
have finite first moments, then for every fixed $x$,
\[
  V^{(\tau),\mathrm{Lap}}_{p,q}(x)\longrightarrow\Delta\mu .
\]
If they have finite second moments, then
\[
  \Norm{V^{(\tau),\mathrm{Lap}}_{p,q}(x)-\Delta\mu}_2
  \le\frac{C_{\mathrm{Lap}}(x,p,q)}{\tau},
\]
where
\[
  C_{\mathrm{Lap}}(x,p,q)
  :=\E\!\left[
    \big(\norm{x-Y^+}_2+\norm{x-Y^-}_2\big)
    \norm{Y^+-Y^-}_2\right]<\infty.
\]
The corresponding ratio-normalized population field has the same limit.
\end{proposition}

The proof is deferred to \cref{app:collapse-proofs}.

For both the Gaussian and Laplace similarities, the large-bandwidth limit
turns a distributional comparison into a first-moment comparison. Thus a
flat kernel can destroy the ability of the drift field to see multimodal or
higher-order differences between $p$ and $q$. This explains why raw drift
loss can become small in a regime where Wasserstein or MMD errors remain
large.

\section{Experiments}\label{sec:experiments}

The experiments test the paper's inverse-problem claims rather than
generative-model quality.  They examine cross-bandwidth calibration,
held-out coefficient coverage, operator/residual/abstention behavior, the
benefit of variance adaptation, portability to the Laplace numerator, a
joint basis-size/dimension stress path, Gram-guided probe design, and the
predicted degeneracies.  All basis distributions are Gaussian mixtures with
component covariance $0.25I_d$, where $I_d$ is the $d\times d$ identity
matrix; the end-to-end audits use both Gaussian-RBF and
Laplace interactions.  The mismatch $c$ is available exactly.  For the
Gaussian-RBF interaction, the population observation matrix $M$ and
Gaussian-probe Gram matrix are available in closed form for diagnostic
comparisons.  Laplace population comparisons instead use independent Monte
Carlo references that never enter the reported certificate.

We distinguish three roles.  A \emph{population} quantity is evaluated from
the closed-form model; a \emph{statistical} quantity is computed from stated
random samples; and a \emph{plug-in} quantity omits a confidence radius and
has no coverage claim.  Gaussian audit samples are not truncated: the global
envelope of \cref{lem:rbf-envelope} applies directly.  The calibration,
  held-out component study, end-to-end audits, joint stress path, and
  probe-design experiments are reproduced by fixed-seed scripts.  The public
  companion repository provides the scripts, frozen configurations, raw rows,
  validation reports, and figure data.  The end-to-end audits estimate $M$ by
  Monte Carlo and use
the outward-rounded Gram lower bound of \cref{prop:gram-lower}; other
singular values and ranks reported in descriptive experiments remain
floating-point evaluations, not symbolic certificates.

\subsection{Conditioning is not cross-bandwidth distance calibration}
\label{sec:exp-calibration}

The first study uses two two-dimensional mixture pairs: a plus-shaped mixture
against a cross-shaped mixture ($m=9$, $r=36$), and a compact outer ring
against an inner ring ($m=8$, $r=28$).  For both designs $N=40$, so
$dN=80\ge r$.  Every one of the $50$ basis/probe/bandwidth matrices---five
seeds and five bandwidths for each design---has numerical column rank $r$.

For $b_t=(1-t)a+t b_{\mathrm{target}}$, with
$t\in\{0,0.1,\ldots,1\}$, and
$\tau\in\{0.5,2,8,32,128\}$, we compare sliced $W_2$ with the empirical raw
drift and with
\begin{equation}
  C_X:=\frac{\norm{\vecop(\Vhat_X)}_2}{\smin(M)}.
  \label{eq:CX}
\end{equation}
Here $C_X$ is deliberately a radius-free plug-in statistic, not the
certificate of \cref{cor:operational}.  The drift batch has $n=1024$; sliced
$W_2$ uses $4{,}000$ samples and $64$ fixed directions per configuration.

Across all $550$ rows, Spearman correlation with sliced $W_2$ is $0.345$ for
raw drift and $0.165$ for $C_X$; restricting to $t>0$ gives $0.337$ and
$0.155$.  These pooled values test comparability across bandwidths, not
generic ordering at a fixed bandwidth: because
$a\wedge b_t=t(a\wedge b_{\mathrm{target}})$, the raw and conditioned
population drifts are exactly linear in $t$ for each frozen
design and $\tau$. Restricting the following within-configuration and
within-family summaries to $t>0$, across the $50$ truly frozen
(dataset, seed, bandwidth) configurations the empirical raw drift and $C_X$
have identical ranks because $\smin(M)$ is a fixed positive scale. Their
Spearman correlations with sliced $W_2$ therefore coincide: both range from
$0.200$ to $1.000$, with median $0.909$. If the five seeds are instead
pooled within each of the ten dataset/bandwidth families, the correlations
range from $0.527$ to $0.943$ for empirical raw drift and from $0.371$ to
$0.976$ for $C_X$; these are pooled family summaries, not fixed-design
results. The supported conclusion is that neither statistic exhibits a
bandwidth-invariant Wasserstein calibration in these benchmarks. Dividing by a
worst-direction singular value produces a conservative inverse scale, not a
generic distance estimator.  The smallest observed numerical ratio
$\smin(M)/\smax(M)$ is $6.73\times10^{-11}$, where
$\smax(M):=\sup_{\norm{u}_2=1}\norm{Mu}_2$ is the largest singular value.
This small ratio explains the very large
plug-in values in some configurations.

The cross-bandwidth plots and the corresponding comparison with the exact
mismatch appear in
\cref{fig:calibration-sw,fig:calibration-c}.  The latter shows how a nearly
singular worst direction can inflate the radius-free plug-in.  The next
experiment incorporates the sampling-error radius.

\subsection{Held-out component coverage and TV usefulness}
\label{sec:exp-operational}

We next isolate the coefficient-bound component of
\cref{cor:operational} in an exactly specified $m=2$ model.  In one
dimension, the component means are $(-1,1)$,
$a=(0.75,0.25)$, $b=(0.25,0.75)$, and hence $\abs{c}=0.5$.  We use
$\tau=2$, $N=20$ probes drawn once per seed from $\mathcal N(0,1)$ and then
frozen, and independent held-out batches of sizes
$n\in\{64,128,256,512,1024,2048\}$.  The observation matrix is closed form,
so the theoretical exact-$M$ bound has $\varepsilon_M=0$, and the exact-basis
setting gives zero residual radius.  The plots evaluate its singular value by
floating-point SVD; they are therefore a numerical evaluation of that
theoretical bound, not an interval-arithmetic certification of
$\underline\sigma$.  At confidence level $1-\delta=0.95$, we compare the empirical-Bernstein,
bounded-vector, and coordinate-Hoeffding radii of
\cref{lem:drift-radii}(a)--(c).  The envelope is the
rigorous value $B_\infty=\sqrt{\tau/e}=0.857764$.
Each method is evaluated as a separately prespecified procedure; we do not
select the smallest radius after observing a batch.

There are $400$ independent audit batches for each of five frozen probe sets
at every $n$, for $12{,}000$ rows total.  All three error-augmented coefficient
bounds cover $\abs{c}$ in every observed repetition.  The observed coverage
rate of the plug-in alone ranges from $0.492$ to $0.521$ across batch sizes,
close to one half; omitting a radius is therefore not a harmless
approximation.  The lower endpoint of the
probe-set-specific two-sided $95\%$ Wilson interval is $0.99049$ for each of
the three rigorous methods.

The actual TV is
$\tfrac12(2\Phi(2)-1)=0.47725$, while the exact-basis output is
$U_{\mathrm{TV}}=\min\{1,U_c\}$, where $\Phi$ is the standard-normal
cumulative distribution function.  The median rigorous output is $1$ at
$n=64$ and $128$ for all methods.  At $n=256$, the median values are
$0.992$ (empirical Bernstein), $0.978$ (bounded vector), and $1.000$
(Hoeffding), with informative-certificate rates $0.573$, $0.679$, and
$0.453$, respectively.  At $n=2048$ their median outputs fall to $0.606$,
$0.668$, and $0.678$.  Thus this component study demonstrates conservative
coverage and the batch size needed for nontrivial TV output; it does not
claim nominally calibrated coverage. \Cref{fig:operational-coverage}
summarizes these results.

\begin{figure}[H]
  \centering
  \includegraphics[width=\textwidth]{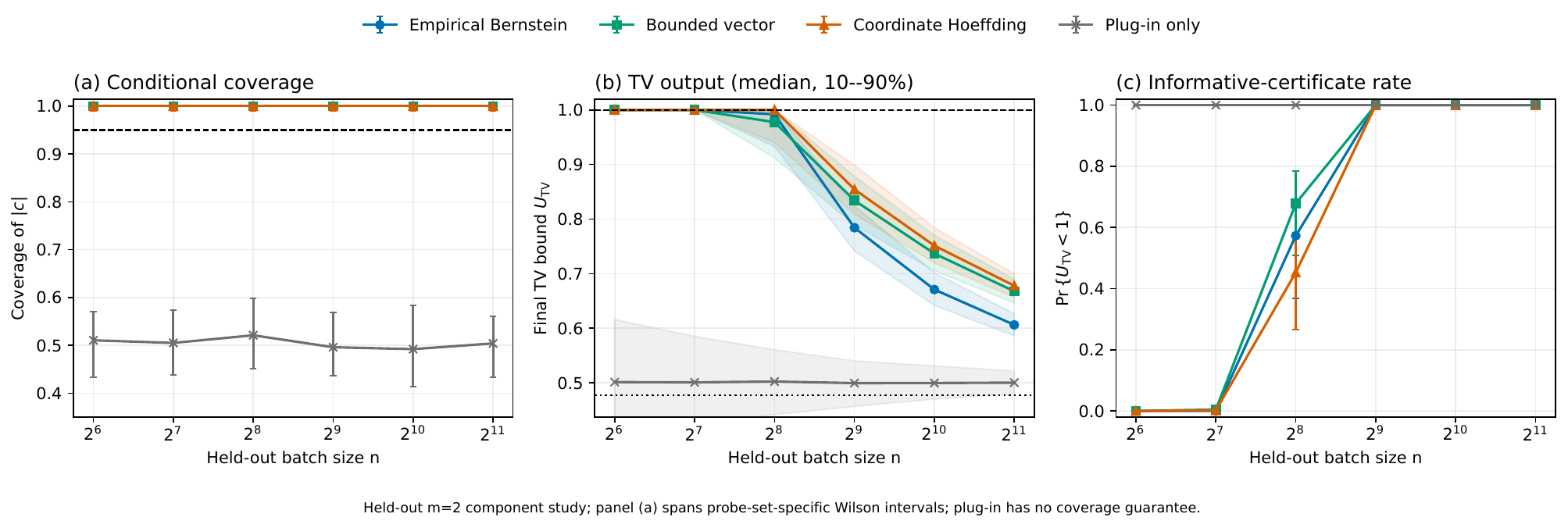}
  \caption{Held-out exact-model component study. (a) Conditional coefficient
  coverage;
  error bars span the five probe-set-specific Wilson intervals and the dashed
  line is $0.95$. (b) Final TV output, with median and $10$--$90\%$ range;
  the dotted line is the true TV and the dashed line is the trivial bound.
  (c) Fraction of audits with $U_{\mathrm{TV}}<1$, with ranges across probe
  sets. The
  three error-augmented coefficient bounds use the global envelope and the
  floating-point singular value of the closed-form $M$; the plug-in has no
  coverage guarantee.}
  \label{fig:operational-coverage}
\end{figure}

\subsection{Variance-adaptive Gaussian and Laplace end-to-end audits}
\label{sec:exp-full-audit}

We next instantiate every term of \cref{cor:operational}.  The basis is the
$m=3$ equilateral Gaussian mixture with means
\[
  (-1,-1/\sqrt3),\quad(1,-1/\sqrt3),\quad(0,2/\sqrt3),
\]
$a=(1/3,1/3,1/3)$,
$b=(0.283\overline3,0.383\overline3,1/3)$, and $\tau=2$.
To make the residual budget nonzero, each full law has contamination
$\zeta=0.001$: $p$ mixes in $\mathcal N((-4,0),0.25I_2)$ and $q$ mixes in
$\mathcal N((4,0),0.25I_2)$.  We declare the valid external radii
$\rho_p=\rho_q=2\zeta=0.002$. Grid quadrature gives
$\TV(p,q)\approx0.04868$ for evaluation only.

For each of $100$ repetitions, each of the three columns of $M$ is estimated
from $s=50{,}000$ independent basis pairs.  We split total failure
probability $0.05$ as $\delta_V=\delta_M=0.025$ and $\delta_R=0$; the
mixture construction deterministically proves the declared residual radii.
Held-out batch sizes are
$n\in\{4{,}096,16{,}384,65{,}536\}$.  The ``pair-midpoint'' design places
one probe at each pair midpoint.  The negative control repeats the origin
three times, so its population operator has rank at most two and a correct
audit must abstain.

For each kernel, we evaluate two separately prespecified procedures on the
same stored operator and held-out samples: bounded-vector radii from
\cref{lem:drift-radii}(b) and coordinate empirical-Bernstein radii from
\cref{lem:drift-radii}(c).  No post-hoc minimum is taken.  For the
empirical-Bernstein operator calibration, column $\alpha$ receives failure
probability $\delta_M/r$.  If $e_\alpha$ is its resulting Euclidean radius,
a union bound and
$\norm{\widehat M-M}_{\rm op}\le\norm{\widehat M-M}_F$ give
\[
  \varepsilon_M^{\rm EB}
  =\left(\sum_{\alpha=1}^{r}e_\alpha^2\right)^{1/2}.
\]
The bounded-vector operator radius is \cref{eq:epsM-MC}.  Both constructions
use the appropriate global Gaussian or Laplace envelope.

Each stored matrix receives the 100-decimal-digit, directed-rounding Gram
lower bound of \cref{prop:gram-lower}.  We separately add a $10^{-9}$ drift
allowance and a $5\times10^{-10}$ operator allowance for ordinary
double-precision evaluation.  These allowances do not interval-enclose every
arithmetic operation, so the study is an implementation demonstration rather
than a formally verified floating-point proof.  Gaussian population checks
use float64 evaluations of closed-form objects.  Laplace checks use an
independent $10^6$-sample Monte Carlo reference with its own error budget and
are descriptive only; neither reference enters the certificate.
\Cref{fig:full-end-to-end} shows the bound distributions, and
\cref{tab:full-audit} reports the largest-batch outcomes.

\begin{figure}[H]
  \centering
  \includegraphics[width=\textwidth]{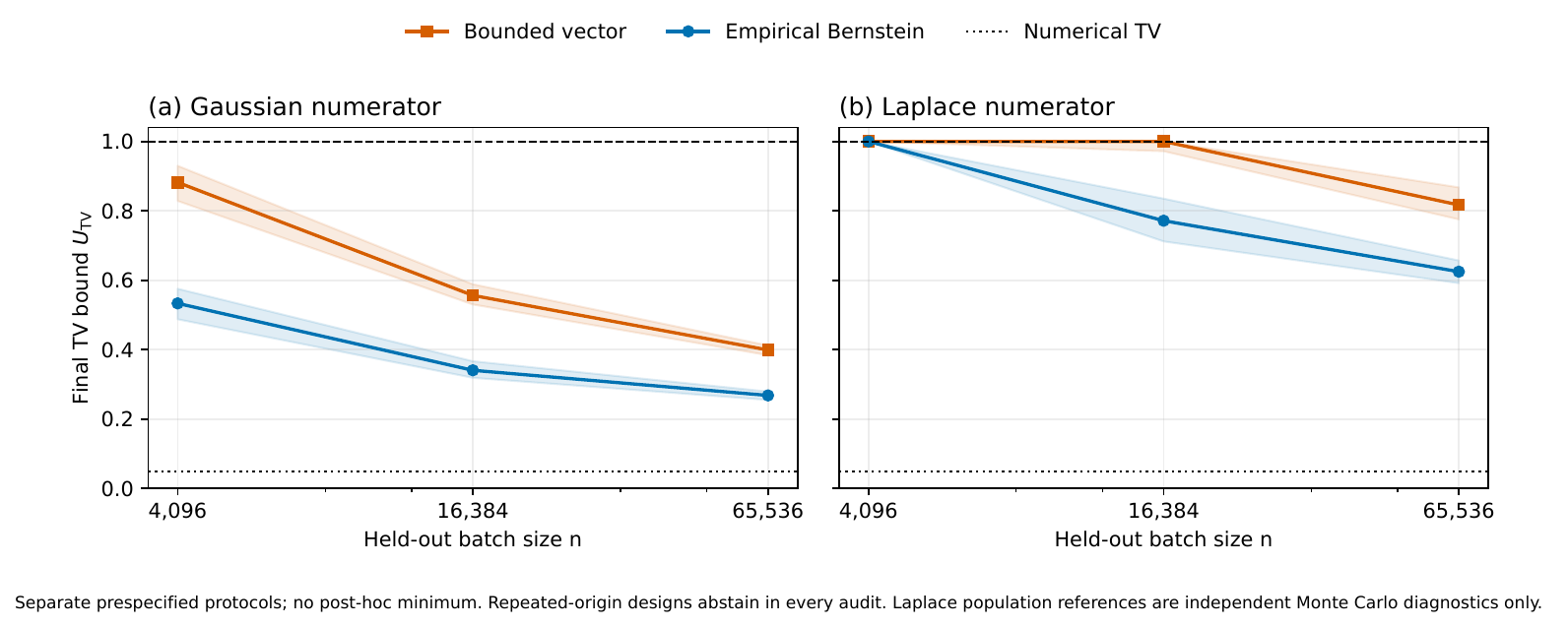}
  \caption{Variance-adaptive end-to-end audits for the observable
  pair-midpoint design.  Curves show median final TV bounds and bands show the
  $10$--$90\%$ range over $100$ repetitions.  The dotted line is the numerical
  TV reference and the dashed line is the trivial bound.  Repeated-origin
  controls are not plotted because every procedure abstains in every audit.}
  \label{fig:full-end-to-end}
\end{figure}

\begin{table}[H]
\centering
\scriptsize
\begin{tabular}{@{}llrlrr@{}}
\toprule
Kernel & Procedure & median $U_{\rm TV}$ & $10$--$90\%$
& informative & reject $H_0$ \\
\midrule
Gaussian & Bounded vector & 0.399 & 0.384--0.412 & 1.00 & 1.00 \\
Gaussian & Empirical Bernstein & 0.268 & 0.255--0.279 & 1.00 & 1.00 \\
Laplace & Bounded vector & 0.817 & 0.775--0.868 & 1.00 & 0.00 \\
Laplace & Empirical Bernstein & 0.625 & 0.591--0.657 & 1.00 & 0.00 \\
\bottomrule
\end{tabular}
\caption{Pair-midpoint outcomes at $n=65{,}536$ over $100$ repetitions.
``Informative'' means $U_{\rm TV}<1$, and ``reject $H_0$'' refers to the
prespecified null $H_0:\TV(p,q)\ge0.5$.  All four procedures non-abstain and
cover both $\norm{c}_2$ and the numerical TV reference in all repetitions.}
\label{tab:full-audit}
\end{table}

For the Gaussian numerator, empirical Bernstein lowers the largest-$n$
median by $32.8\%$.  At $n=16{,}384$ it rejects the equivalence null in all
repetitions, whereas the bounded-vector procedure rejects it in none.  For
the Laplace numerator, empirical Bernstein lowers the largest-$n$ median by
$23.5\%$ and is informative in every repetition already at $n=16{,}384$,
compared with $28\%$ for bounded vector.  Neither Laplace procedure rejects
the equivalence null at the largest $n$.  Equal numerical $\tau$ does not
represent equal effective bandwidth across kernel families, so this is a
portability check rather than a kernel ranking.

Every positive-design configuration has $100/100$ observed coefficient and
numerical-TV coverage; the lower endpoint of the corresponding two-sided
$95\%$ Wilson interval is $0.963$.  This is simulation evidence, not a claim
of exact nominal calibration.  Every repeated-origin procedure has
$\underline\sigma=0$, abstains, and returns $U_{\rm TV}=1$ in all
repetitions.  Even the tighter Gaussian median remains about $5.5$ times the
numerical TV reference, so variance adaptation reduces but does not eliminate
the price of worst-case envelopes.

\subsection{Joint basis-size and dimension stress path}
\label{sec:exp-basis-stress}

To test whether the numerical certificate remains operational beyond $m=3$,
we use regular-simplex Gaussian bases with
$m\in\{3,5,8\}$, dimension $d=m-1$, pairwise mean distance two, and
covariance $0.25I_d$.  Every pair midpoint is a probe, so
$N=r=\binom{m}{2}$, and $\tau=2$.  The basis weights are initially uniform;
$b$ is formed by transferring mass $0.05$ from the first component to the
second.  Consequently, throughout the path,
\[
  \TV(p,q)=0.05\{2\Phi(2)-1\}=0.047725.
\]
We evaluate the theoretical exact-operator specialization using the
closed-form Gaussian $M$, a float64 singular value, and the
empirical-Bernstein drift radius at level $0.95$.  For each $m$, $100$
repetitions are run at
$n\in\{4{,}096,16{,}384,65{,}536,262{,}144\}$.
\Cref{fig:basis-scaling-stress} summarizes the resulting stress path.

\begin{figure}[H]
  \centering
  \includegraphics[width=\textwidth]{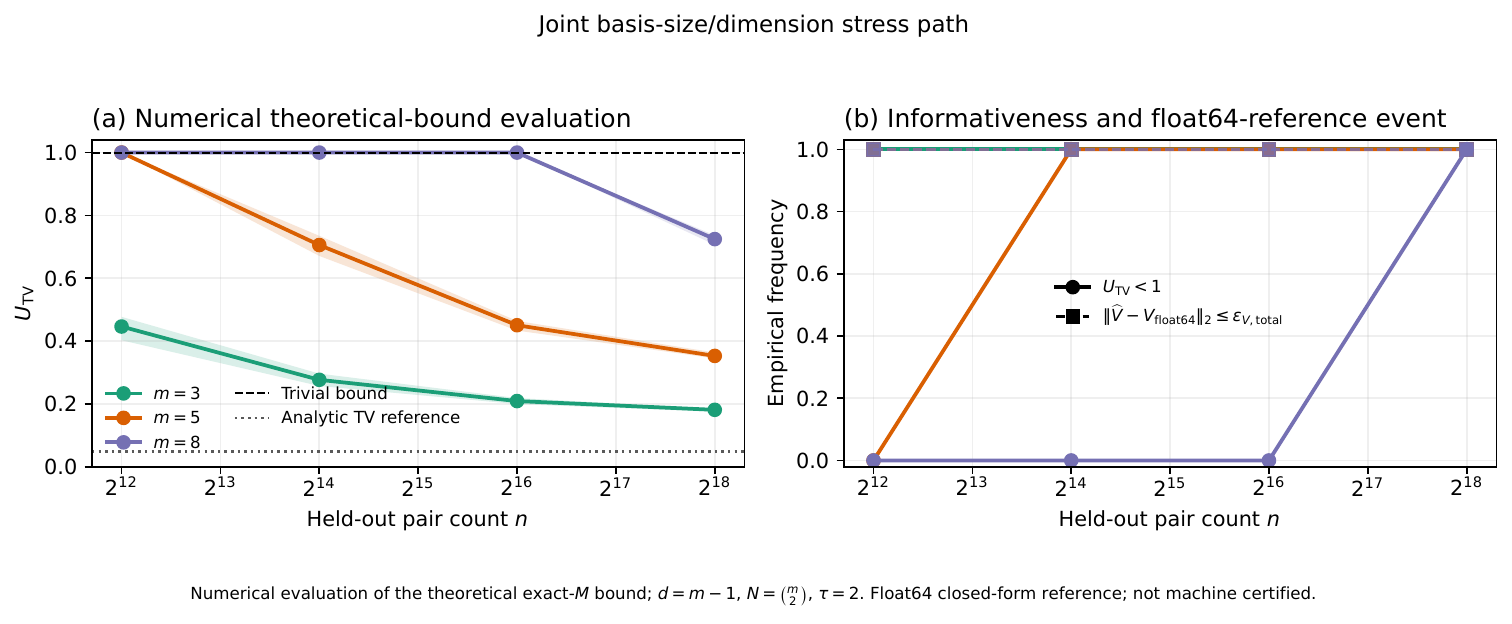}
  \caption{Joint basis-size/dimension stress path.  Left: median
  $U_{\rm TV}$ with $10$--$90\%$ bands.  Right: observed informativeness and
  the float64-reference drift event.  All coefficient and analytic-TV
  references are covered in every repetition.  The displayed exact-$M$
  bounds and singular values are numerical evaluations, not
  machine-certified quantities.}
  \label{fig:basis-scaling-stress}
\end{figure}

At $n=262{,}144$, the median bounds are $0.181$, $0.353$, and $0.725$ for
$m=3,5,$ and $8$, respectively, and every repetition is informative.  The
$m=5$ certificate becomes informative at $n=16{,}384$, whereas the $m=8$
certificate remains trivial through $n=65{,}536$.  Along the same path, the
float64 smallest singular value decreases from $0.253$ to $0.118$ and the
condition number increases from $2.54$ to $4.85$.  All $1{,}200$ drift
events and coefficient/TV reference checks hold; for each $100/100$
frequency, the lower endpoint of the two-sided $95\%$ Wilson interval is
$0.963$.

This study shows operation beyond the three-component example while exposing
a steep finite-sample cost.  It is not an isolated or asymptotic scaling law
in $m$: the ambient dimension, coefficient dimension, probe count,
conditioning, and TV bridge factor $\beta_\phi=m/2$ all change together.

\subsection{The population Gram matrix as a probe-design objective}
\label{sec:exp-probedesign}

For the same $m=3$ equilateral Gaussian basis and $\tau=2$, we vary the standard
deviation $s$ of a centered Gaussian probe law over ten values from $0.15$
to $4$.  For each $s$, the closed form in \cref{app:randprobes} gives
$\sqrt{\gamma(\nu)}$, evaluated by floating-point eigendecomposition.  We compare it with
$\smin(M)/\sqrt N$ for $N\in\{10,40,160,640\}$ over $100$ random probe
realizations per design.

The best value on the prespecified grid is $s=0.5$, with
$\sqrt{\gamma(\nu)}=0.119347$, but $s=0.75$ is nearly tied at $0.119039$;
the result identifies a broad useful region rather than a sharp optimum.
At $s=0.5$ and $N=640$, the median ratio
$(\smin(M)/\sqrt N)/\sqrt{\gamma(\nu)}$ is $0.99870$.  Across all ten probe
spreads, the corresponding $N=640$ median ratios range from $0.976$ to
$1.000$.  The event
$\smin(M)\ge\sqrt{N\gamma(\nu)/2}$ from \cref{thm:randprobe} occurs in
$17$--$99\%$ of cells at $N=10$, $50$--$100\%$ at $N=40$,
$93$--$100\%$ at $N=160$, and every cell at $N=640$.  Thus the population
criterion identifies a useful probe scale and predicts the large-$N$
singular-value scale in this benchmark; finite-$N$ spread remains substantial
for diffuse probe laws. \Cref{fig:probe-law-design} displays both the
cross-design comparison and convergence at the selected grid optimum.

\begin{figure}[H]
  \centering
  \includegraphics[width=0.92\textwidth]{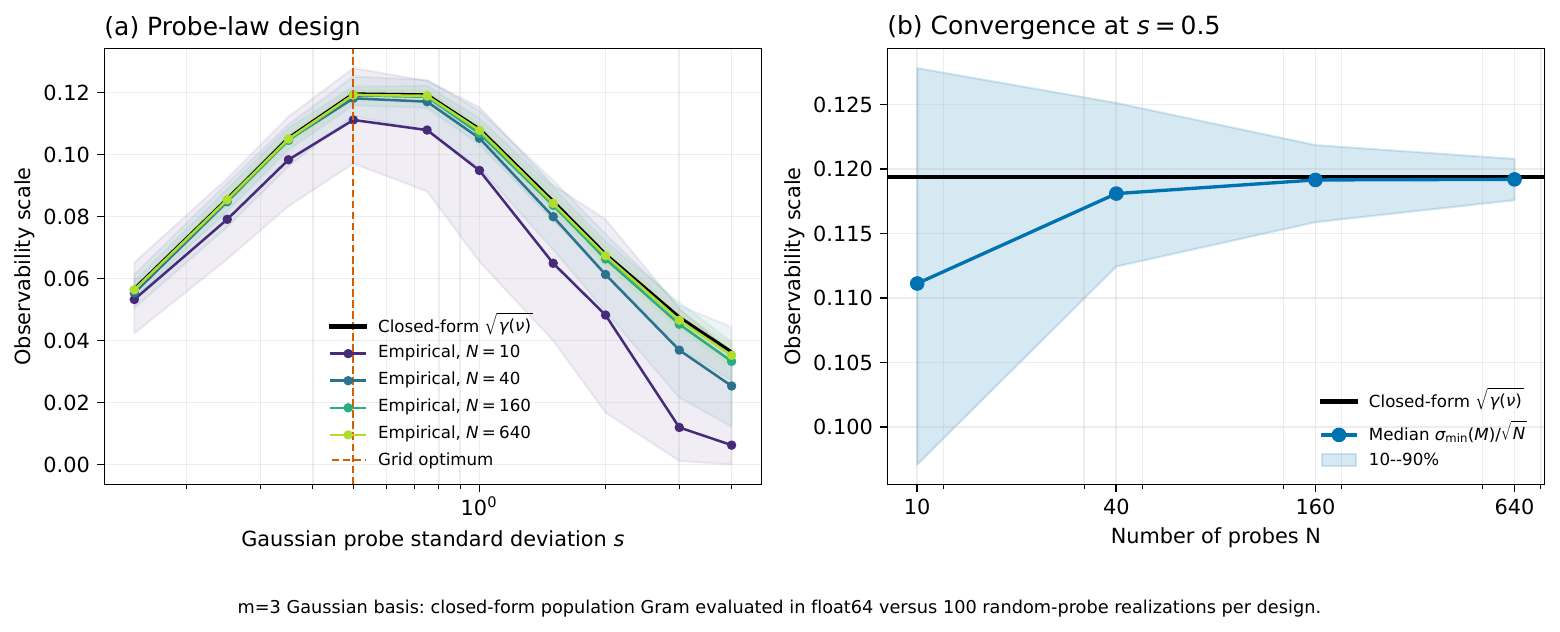}
  \caption{Probe-law design. Left: closed-form population scale evaluated numerically,
  $\sqrt{\gamma(\nu)}$ and empirical $\smin(M)/\sqrt N$ over Gaussian probe
  spreads; bands show the $10$--$90\%$ range over $100$ realizations. Right:
  convergence at the grid-optimal spread $s=0.5$.}
  \label{fig:probe-law-design}
\end{figure}

\newpage

\subsection{Large-bandwidth collapse}\label{sec:exp-bandwidth}

The bandwidth study uses the mean-matched pair
$a=(1/2,0,1/2)$ and $b=(0,1,0)$ at means
$(-1.5,0),(0,0),(1.5,0)$.  The laws differ, but their means agree.
We use $N=40$ probes from $\mathcal N(0,4I)$.  Over
$\tau\in[10^{-1},10^5]$, the population drift
$\norm{V_X}_F=\norm{Mc}_2$ is evaluated from the closed form.  The fitted
log--log slope over the largest decade and a half is $-0.9975$ on average
across five probe seeds, with range $[-0.9979,-0.9971]$.  This benchmark
exhibits numerical decay consistent with the $O(1/\tau)$ upper rate in
\cref{rem:taurate}; the lemma itself does not assert a matching lower bound.
Empirical drift reaches a sampling-noise floor while the population inverse
bound remains above $\norm{c}_2$.  The sliced $W_2$ and MMD diagnostics
recorded in the released per-seed data are computed once per seed and then
held fixed across the $\tau$ sweep.  The MMD evaluation kernel is
$k_{\rm eval}(x,y)=\exp(-\norm{x-y}_2^2/h)$, where $h$ is the median pooled
pairwise squared distance from $2{,}000$ samples per law for that seed; it is
not the drifting kernel $k_\tau$. \Cref{fig:bandwidth} displays the drift,
sampling-floor, and conditioning diagnostics.

\begin{figure}[H]
  \centering
  \begin{minipage}{0.32\textwidth}
    \centering
    \includegraphics[width=\linewidth]{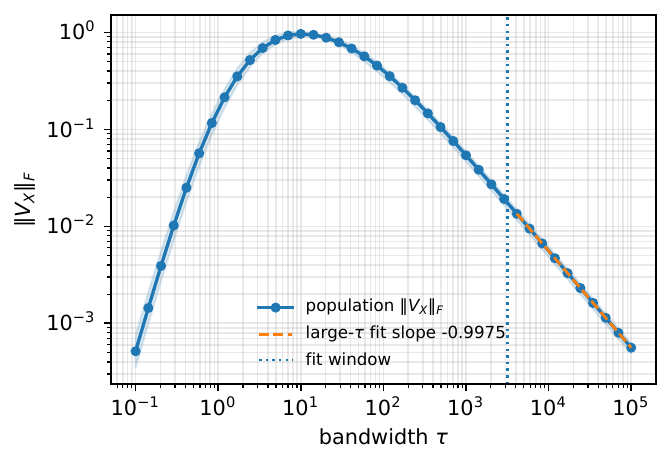}
  \end{minipage}\hfill
  \begin{minipage}{0.32\textwidth}
    \centering
    \includegraphics[width=\linewidth]{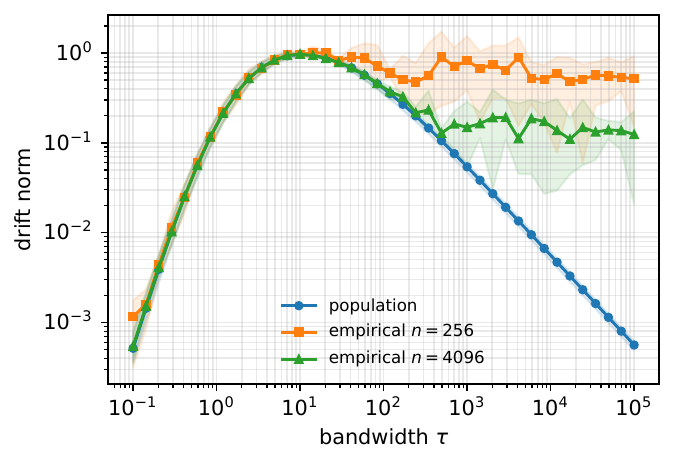}
  \end{minipage}\hfill
  \begin{minipage}{0.32\textwidth}
    \centering
    \includegraphics[width=\linewidth]{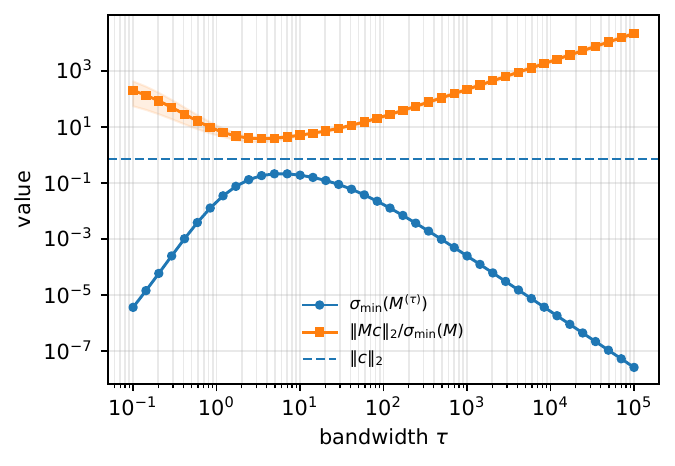}
  \end{minipage}
  \caption{Large-bandwidth collapse for a mean-matched pair. Left: closed-form
  population drift with slope near $-1$. Middle: mean population and
  empirical drift at two batch sizes; matching-color bands span the min--max
  range over the five probe seeds. Right: numerical $\smin(M^{(\tau)})$ and the
  population inverse bound $\norm{Mc}_2/\smin(M)$, with exact $\norm{c}_2$ as
  reference.}
  \label{fig:bandwidth}
\end{figure}

\subsection{Rank boundaries and structural collisions}\label{sec:exp-rank}

The final diagnostics separate floating-point ill-conditioning from exact
structural failure.  For general-position Gaussian means, the ambient-rank
study tests \cref{conj:minprobes}.  Matrices are column normalized, and
numerical rank uses NumPy's default tolerance
$\sigma_1\max\{dN,r\}\epsilon_{\rm mach}$, where
$\epsilon_{\rm mach}$ denotes float64 machine precision. A row with
$10^{-12}\le\sigma_\rho/\sigma_1\le10^{-8}$ and
$\rho=\min\{dN,r\}$ is declared ambiguous.  Of $290$ double-precision rows,
$250$ match the predicted rank, $37$ are ambiguous, and three fall below the
default tolerance.  High-precision outcomes are reported separately rather
than converted into exact-rank claims.  Recomputing all $40$ boundary rows at
$60$ and $100$ decimal digits yields a positive ratio in every case, from
$8.15\times10^{-18}$ to $8.81\times10^{-9}$; the largest relative change
between precisions is $1.05\times10^{-27}$.  This is strong numerical support,
not an exact-rank proof.

The full configuration-by-configuration breakdown is reported in
\cref{tab:rank-summary}. Thresholded double-precision rank is a numerical
diagnostic and does not establish exact rank.

Finally, \cref{fig:midpoint} reproduces the exact midpoint-collision mechanism
of \cref{ex:midpoint}.  For $\theta=(0,1,2,3)$, the witnesses for pairs
$(1,4)$ and $(2,3)$ are proportional to relative discrepancy
$2.38\times10^{-15}$ on the evaluation grid, and the population Gram matrix
has relative smallest eigenvalue $2.86\times10^{-17}$.  Moving the final
mean to $3.15$ breaks the collision and gives
$\gamma(\nu)=1.19\times10^{-8}$ for
$\nu=\mathcal N(0,4)$. As \cref{prop:midpoint-infeasible}
shows, the displayed formal null line is not itself a feasible coefficient
mismatch; the experiment diagnoses failure of the ambient inverse
certificate, not a pair of distinct indistinguishable distributions.

\begin{figure}[H]
  \centering
  \includegraphics[width=0.58\textwidth]{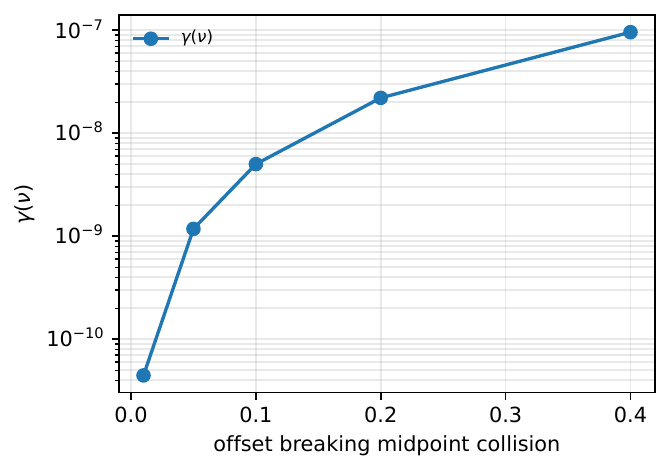}
  \caption{Midpoint collision and symmetry breaking. The symmetric basis has
  exact $\gamma(\nu)=0$ and a formal null direction for every probe set; zero
  offset is omitted from the logarithmic horizontal scale. Perturbing one
  mean separates the pair midpoints and restores a positive population
  observability scale.}
  \label{fig:midpoint}
\end{figure}

\paragraph{Reproducibility.}
The script \texttt{run\_operational\_experiments.py} regenerates
the calibration, held-out component study, original bounded-vector
end-to-end audit, and probe-law study,
their CSV files,
and all figures used in
\cref{fig:calibration-sw,fig:calibration-c,fig:operational-coverage,%
fig:probe-law-design}.
It fixes a master seed,
asserts all required matrix ranks, checks the global Gaussian envelope on
every generated interaction in the held-out component study, and records
every repetition.  The script
\texttt{run\_strengthening\_experiments.py} regenerates and validates the
Gaussian/Laplace audits and joint stress path underlying
\cref{fig:full-end-to-end,fig:basis-scaling-stress}; its procedures,
configurations, and stopping gates are frozen in
\texttt{strengthening\_protocol.json}.  The companion
\texttt{run\_retained\_diagnostics.py} regenerates the rows and figures
underlying \cref{fig:bandwidth,tab:rank-summary,fig:midpoint} from their
fixed configurations, while
\texttt{recompute\_rank\_boundaries\_high\_precision.py} reproduces all
$40$ boundary rows at $60$ and $100$ decimal digits.  The experimental code,
configurations, per-repetition outputs, and validation reports are available
in the accompanying public repository.

\section{Limitations and Discussion}\label{sec:discussion}

The practical lesson is that small held-out drift is informative only after
three quantities are declared: the finite model or residual class, a lower
confidence bound on probe observability, and a sampling-error budget.
\Cref{cor:operational} combines them into a total-variation upper
confidence bound and abstains when the calibrated observability margin is
nonpositive.  This is why the empirical plug-in in
\cref{sec:exp-calibration} cannot be interpreted as a certificate, even when
its denominator is computed from the exact population matrix.

The primary object controlled by the inverse step is $c$ (or $c_m$).  The
exact operator-norm identity underlying \cref{lem:bridge} converts this to
total variation within the declared basis; \cref{cor:fulltv} adds explicit
$L^1$ residuals for the full distributions.  Bounded IPMs, including
bounded-kernel MMD, inherit a conservative guarantee through
\cref{rem:ipm}.  No Wasserstein or KL bound follows from this TV argument
without additional support, moment, or density-ratio assumptions.  The
operational bound
also induces the level-$(\delta_V+\delta_M+\delta_R)$ TV-equivalence test of
\cref{cor:operational-equivalence}: a failure to certify remains
inconclusive rather than evidence that $p\ne q$.

Bandwidth and probe choice define the measurement system rather than acting
as ordinary tuning parameters.  The population Gram matrix predicts which
probe laws are informative, and \cref{sec:exp-probedesign} shows that its
smallest eigenvalue accurately predicts the large-$N$ singular-value scale
in the Gaussian benchmark.  The midpoint and large-bandwidth examples expose
two different failures.  A midpoint symmetry creates an ambient blind
direction for every probe set, although feasibility must still be checked;
a flat kernel instead collapses toward low-order moment comparison at
$O(1/\tau)$ on fixed probes under the stated higher-moment assumptions.
Poor but nonzero conditioning lies between
these extremes and yields valid yet weak bounds.
The end-to-end experiments sharpen this point: empirical-Bernstein radii
materially reduce slack without changing the certificate, but the tightest
Gaussian median remains about $5.5$ times the numerical TV reference.  The
joint stress path remains informative through $m=8$ only at the largest
tested batch size, illustrating rather than resolving the basis-size cost.

Several limitations are structural.  The guarantee is conditional on a
finite basis or a valid residual envelope, and on independence between the
design choices and the evaluation batch.  It applies directly to the
unnormalized numerator field; the implemented audit recomputes that numerator
from fresh samples.  Ratio-normalized fields need a calibrated joint
numerator--denominator event beyond \cref{rem:normalized}.  It does not cover an adaptive training
loss that reuses the same samples as probes and negatives.  Ambient full
column rank is sufficient but conservative because feasible mismatches form
a lower-dimensional algebraic set. For the probability-density-basis bound
used in the experiments, the TV conversion constant is $m/2$ and therefore
grows with the basis size, so finer resolution weakens the worst-case
per-unit bound.
The density formulation excludes singular generator pushforwards and
empirical measures.  Finally, the synthetic experiments
evaluate the audit mechanism and its predicted degeneracies; they do not
establish image-generation quality or universal behavior outside the
benchmark classes.

\paragraph{External residual radii.}
The drift evaluations do not determine $\rho_p$ or $\rho_q$.  For example,
if $\widetilde p=(1-\zeta_p)p_m+\zeta_p h_p$, where $p_m,h_p$ are probability
densities and $\zeta_p\in[0,1]$, and an external analysis certifies
$\norm{p-\widetilde p}_{L^1}\le\eta_p$, then the triangle inequality gives
$\norm{p-p_m}_{L^1}\le\eta_p+2\zeta_p$; hence
$\rho_p=\eta_p+2\zeta_p$ is valid, and analogously for $q$.  An in-sample fit
score alone gives no certificate, and any stochastic validation failure
must enter the audit's joint error budget.

Natural extensions include restricted observability over feasible
mismatches, confidence sequences for repeated audits, statistically valid
adaptive probe design, and residual envelopes for concrete approximation
spaces.  A target-aware refinement could propagate the existing coefficient
confidence region directly to the coefficient gap $a-b$, rather than first
controlling the full mismatch norm through the worst-direction scale
$\smin(M)$; constructing a calibrated version under estimated-operator
uncertainty is left for future work.  Neural optimization analyses could
also connect generator training dynamics to the conditioning of the field
observed by the audit.

\paragraph{Reproducibility statement.}
All theoretical claims are accompanied by proofs in the main text or
appendices.  Code and numerical outputs are available at
\url{https://github.com/SamAndersson-C/finite-drift-certification}.
The repository contains fixed-seed code for the calibration,
held-out component study, Gaussian/Laplace end-to-end audits, joint stress
path, and probe-law experiments, plus deterministic runners for the retained
bandwidth, rank, high-precision, and midpoint diagnostics.  It includes the
prespecified strengthening protocol, per-repetition CSV outputs, a validator
that reconstructs every bound reported from the E9/E10 strengthening
experiments, and vector figures.  Probe locations
are fixed before independent evaluation batches are drawn.  The experiments
use only synthetic distributions and require Python, NumPy, SciPy, pandas,
and Matplotlib, with mpmath for the high-precision rank diagnostic; no
external data or pretrained models are used.

\par\smallskip
\noindent{\footnotesize
\textbf{CRediT authorship contribution statement.}
\textbf{Sam Andersson:} Conceptualization (lead), Methodology (lead), Formal
analysis (lead), Investigation (lead), Software (lead), Validation (lead),
Visualization (lead), Project administration (lead), Writing -- original
draft (lead), Writing -- review \& editing.
\textbf{Ricky Molén:} Methodology (supporting; discussions concerning
drifting models and the identifiability ideas developed in the precursor
group project), Investigation (supporting; preliminary experiments conducted
within the precursor group project), Writing -- review \& editing.
\par}

\par\smallskip
\noindent{\footnotesize
\textbf{Declaration of generative AI use.}
During the preparation of this manuscript, the authors used large language
model tools for language editing, LaTeX troubleshooting, feedback on
exposition and organization, literature searches,
mathematical arguments, and assistance with implementing and
debugging code for the numerical experiments. The authors reviewed and edited
all AI-assisted material and independently verified all proofs,
code, and numerical results. They take full responsibility for the content of
the manuscript.
\par}

\bibliographystyle{tmlr}
\bibliography{references}

\appendix
\crefalias{section}{appendix}
\raggedbottom
\section{Deferred Algebraic Details}\label{app:algebra}

\subsection{Antisymmetry gives equilibrium}\label{app:antisym-equilibrium}

Assume the kernel antisymmetry of \cref{def:antisym} and absolute
integrability. Swapping the dummy variables in the field with reversed
arguments gives
\[
\begin{aligned}
  V_{q,p}(x)
  &=\int\!\!\int K(x,y^+,y^-)q(y^+)p(y^-)\dd y^+\dd y^-\\
  &=\int\!\!\int K(x,y^-,y^+)p(y^+)q(y^-)\dd y^+\dd y^-\\
  &=-V_{p,q}(x).
\end{aligned}
\]
If $p=q$, reversing the two labels does not change the field, so
$V_{p,p}(x)=-V_{p,p}(x)$. Hence $V_{p,p}(x)=0$ in the real vector space
$\R^d$.

\subsection{Antisymmetry of the pair responses}
\label{app:U-antisym}

Fix a probe $x_\ell$. By definition,
\[
  U_{ji}[:, \ell]
  = \int\!\!\int K(x_\ell, y^+, y^-)\, \phi_j(y^+) \phi_i(y^-)
    \dd y^+ \dd y^-.
\]
Swapping dummy integration variables gives
\[
  U_{ji}[:, \ell]
  = \int\!\!\int K(x_\ell, y^-, y^+)\, \phi_j(y^-) \phi_i(y^+)
    \dd y^+ \dd y^-.
\]
By kernel antisymmetry, $K(x_\ell, y^-, y^+) = -K(x_\ell, y^+, y^-)$.
Therefore
\[
  U_{ji}[:, \ell]
  = -\int\!\!\int K(x_\ell, y^+, y^-)\, \phi_i(y^+) \phi_j(y^-)
    \dd y^+ \dd y^-
  = -U_{ij}[:, \ell].
\]
Since this holds for every $\ell$, $U_{ji} = -U_{ij}$.

\subsection{From zero minors to proportional coefficients}
\label{app:minors}

Assume
\[
  a_i b_j - a_j b_i = 0 \quad \text{for every } 1 \le i < j \le m.
\]
Since $q$ is a probability density, $b \ne 0$. Choose $k$ such that
$b_k \ne 0$. For any $i$, the zero-minor relation involving $i$ and $k$
gives $a_i b_k = a_k b_i$. Dividing by $b_k$ gives
$a_i = (a_k / b_k)\, b_i$. Thus $a = \lambda b$, with
$\lambda = a_k / b_k$. The finite-basis expansions imply $p = \lambda q$.
Since both densities integrate to one, $\lambda = 1$, and therefore
$p = q$.

\subsection{Dimension of the decomposable mismatch set}
\label{app:grassmannian}

Ignoring positivity and normalization, the nonzero projective directions of
$c=a\wedge b$ form the Grassmannian $\operatorname{Gr}(2,m)$, of dimension
$2m-4$, inside an ambient projective space of dimension
$\binom{m}{2}-1$. Hence full column rank on all of $\R^{\binom m2}$ can be
substantially stronger than injectivity on the structured mismatch set, as
is typical in algebraic inverse problems
\citep{chandrasekaran2012geometry,breiding2023algebraic}.

\subsection{An \texorpdfstring{$L^1$}{L1}-residual bound}
\label{app:residual}

The residual $R_m$ of \eqref{eq:approx-decomp} can be bounded explicitly
under a simple bounded-kernel assumption. Recall the stacked interaction
vector $K_X$ from \eqref{eq:KX}.

\begin{proposition}[An $L^1$-residual bound]\label{prop:residual}
Assume (A0), (A3), and (A7), and let
$p_m,q_m,r_p,r_q\in L^1(\R^d)$. Thus
$\norm{K_X(y^+, y^-)}_2 \le B_X$ for every
$y^+,y^-\in\R^d$. Then all displayed residual integrals are absolutely
integrable, and the residual in \eqref{eq:approx-decomp} satisfies
\begin{equation}
  \norm{R_m}_2
  \le B_X \Big(
      \norm{r_p}_{L^1}\norm{q_m}_{L^1}
      +\norm{p_m}_{L^1}\norm{r_q}_{L^1}
      +\norm{r_p}_{L^1}\norm{r_q}_{L^1}
    \Big). \label{eq:residual-bound}
\end{equation}
\end{proposition}

\begin{proof}
Using $p = p_m + r_p$ and $q = q_m + r_q$, the product density expands as
\[
  p q = p_m q_m + r_p q_m + p_m r_q + r_p r_q.
\]
The term $p_m q_m$, after antisymmetric grouping, gives $M_m c_m$.
Therefore the residual vector is
\[
  R_m = \int\!\!\int K_X(y^+, y^-)
  \Big( r_p(y^+) q_m(y^-) + p_m(y^+) r_q(y^-) + r_p(y^+) r_q(y^-) \Big)
  \dd y^+ \dd y^-.
\]
Taking norms and applying the triangle inequality gives
\[
  \norm{R_m}_2
  \le \int\!\!\int \norm{K_X(y^+, y^-)}_2
  \Big( \abs{r_p(y^+) q_m(y^-)} + \abs{p_m(y^+) r_q(y^-)}
    + \abs{r_p(y^+) r_q(y^-)} \Big) \dd y^+ \dd y^-.
\]
Using $\norm{K_X(y^+, y^-)}_2 \le B_X$ and factorizing the product
integrals gives \eqref{eq:residual-bound}.
\end{proof}

\section{Deferred Proofs}\label{app:proofs}

This appendix collects routine concentration, perturbation-calibration, and
limit arguments whose statements are used in the main text.

\subsection{Gaussian-RBF interaction envelope}
\label{app:proof-rbf-envelope}

\begin{proof}[Proof of \cref{lem:rbf-envelope}]
Write $u=y^+-x$, $v=y^--x$, and $s=\norm{u-v}_2$. The parallelogram
identity gives $\norm{u}_2^2+\norm{v}_2^2\ge s^2/2$, so
\[
  \norm{K_\tau(x,y^+,y^-)}_2
  \le s\exp\!\left(-\frac{s^2}{2\tau}\right).
\]
The function on the right is maximized at $s=\sqrt\tau$ and has value
$\sqrt{\tau/e}$, proving \eqref{eq:rbf-envelope}. Stacking $N$ vectors and
summing their squared norms gives \eqref{eq:stacked-envelope}.
\end{proof}

\subsection{Laplace interaction envelope}
\label{app:proof-laplace-envelope}

\begin{proof}[Proof of \cref{lem:laplace-envelope}]
Let $s=\norm{y^+-y^-}_2$.  The triangle inequality gives
\[
  \norm{x-y^+}_2+\norm{x-y^-}_2\ge s.
\]
Therefore
\[
  \norm{K_\tau(x,y^+,y^-)}_2
  =s\exp\!\left(
     -\frac{\norm{x-y^+}_2+\norm{x-y^-}_2}{\tau}\right)
  \le s e^{-s/\tau}\le\frac{\tau}{e},
\]
where the last maximum is attained at $s=\tau$.  Stacking $N$ probe
vectors and summing their squared norms gives
$B_{N,\tau}^{\mathrm{Lap}}=\sqrt N\,\tau/e$.
\end{proof}

\subsection{Held-out drift radii}\label{app:radius-proofs}

\begin{proof}[Proof of \cref{lem:drift-radii}(a)]
Fix a coordinate $\rho\in\{1,\dots,D\}$. Since
$(\Xi_s)_\rho\in[-B_\infty,B_\infty]$, Hoeffding's inequality gives
\[
  \PP\big(\abs{\hat z_\rho-z_\rho}\ge t\big)
  \le 2\exp\!\left(-\frac{nt^2}{2B_\infty^2}\right).
\]
Choose $t=B_\infty\sqrt{2\log(2D/\delta)/n}$. The probability that
coordinate $\rho$ violates $\abs{\hat z_\rho-z_\rho}\le t$ is at most
$\delta/D$. A union bound over the $D$ coordinates and
$\norm{u}_2\le\sqrt D\norm{u}_\infty$ yield
\eqref{eq:hoeffding-bound}.
\end{proof}

\begin{proof}[Proof of \cref{lem:drift-radii}(b)]
Let $F:=\norm{\hat z-z}_2$. Independence and the vanishing of cross terms
give
\[
  \E F
  \le\sqrt{\E F^2}
  =\sqrt{\frac1n\E\norm{\Xi_1-z}_2^2}
  \le\frac{B_2}{\sqrt n}.
\]
Changing one input vector changes $F$ by at most $2B_2/n$.
McDiarmid's bounded-difference inequality
\citep{mcdiarmid1989bounded} therefore gives
\[
  \PP\{F-\E F\ge t\}
  \le\exp\!\left(-\frac{nt^2}{2B_2^2}\right).
\]
Taking $t=B_2\sqrt{2\log(1/\delta)/n}$ proves
\eqref{eq:vector-mean-bound}; substituting
$B_2=\sqrt{N\tau/e}$ gives \eqref{eq:epsV-vector}.
\end{proof}

\begin{proof}[Proof of \cref{lem:drift-radii}(c)]
Rescale one coordinate from $[-B_\infty,B_\infty]$ to $[0,1]$ and apply
the empirical Bernstein inequality of
\citet[Theorem~4]{maurer2009empirical} to the rescaled coordinate and its
complement in $[0,1]$. With failure probability $\delta/(2D)$ for each
one-sided inequality, rescaling back gives
$\abs{\hat z_\rho-z_\rho}\le r_\rho(\delta)$. A union bound over the $2D$
one-sided events and the Euclidean norm inequality yield
\eqref{eq:epsV-EB}.
\end{proof}

\subsection{Monte Carlo operator calibration}
\label{app:proof-M-calibration}

\begin{proof}[Proof of \cref{prop:M-calibration}]
The expectation of $W_{\alpha t}$ is the corresponding column $u_\alpha$ of
$M_m$. Apply \cref{lem:drift-radii}(b) to each column with failure
probability $\delta_M/r$ and take a union bound. On the resulting event,
\[
  \norm{\widehat M_m-M_m}_{\mathrm{op}}
  \le\norm{\widehat M_m-M_m}_F
  =\left(\sum_{\alpha\in P}
    \norm{\hat u_\alpha-u_\alpha}_2^2\right)^{1/2},
\]
which gives \eqref{eq:epsM-MC}.
\end{proof}

\subsection{Random-probe conditioning}\label{app:proof-randprobe}

\begin{proof}[Proof of \cref{thm:randprobe}]
By \eqref{eq:gramdecomp}, $M^\top M=\sum_\ell A_\ell$ with
$A_\ell:=G(X_\ell)^\top G(X_\ell)$ independent, identically distributed,
symmetric, and positive semidefinite. By (A9),
\[
  \lmax(A_\ell)=\norm{G(X_\ell)}_{\mathrm{op}}^2\le L^2
  \quad\text{a.s.}
\]
Identical distribution and linearity of expectation give
$\mu_{\min}:=\lmin(\E\sum_\ell A_\ell)=N\gamma(\nu)$. The matrix Chernoff
lower-tail inequality \citep[Theorem~5.1.1]{tropp2015matrix}, applied with
$R=L^2$, states
\[
  \PP\!\left(\lmin\!\left(\sum_\ell A_\ell\right)
    \le(1-\varepsilon)\mu_{\min}\right)
  \le r\left[
    \frac{e^{-\varepsilon}}
         {(1-\varepsilon)^{1-\varepsilon}}
  \right]^{\mu_{\min}/R}.
\]
The bracket is at most $e^{-\varepsilon^2/2}$ because
$-\varepsilon-(1-\varepsilon)\log(1-\varepsilon)\le-\varepsilon^2/2$
on $[0,1)$. Indeed,
$h(\varepsilon):=(1-\varepsilon)\log(1-\varepsilon)+\varepsilon
-\varepsilon^2/2$ satisfies $h(0)=0$ and
$h'(\varepsilon)=-\log(1-\varepsilon)-\varepsilon\ge0$.
This proves the first display of \cref{thm:randprobe}.

For the second statement, take $\varepsilon=1/2$. The failure probability
is at most
\[
  r\exp\!\left(-\frac{N\gamma(\nu)}{8L^2}\right)\le\delta'
\]
under the stated probe-count condition. On the complementary event,
\[
  \lmin(M^\top M)>N\gamma(\nu)/2>0.
\]
Hence $M$ has full column rank and
$\smin(M)=\lmin(M^\top M)^{1/2}\ge\sqrt{N\gamma(\nu)/2}$.
\end{proof}

\subsection{Large-bandwidth collapse}\label{app:collapse-proofs}

\begin{proof}[Proof of \cref{thm:collapse}]
For every fixed $x,y\in\R^d$,
$k_\tau(x,y)=\exp(-\norm{x-y}_2^2/\tau)\to1$ as $\tau\to\infty$.
Therefore
\[
  k_\tau(x,Y^+)k_\tau(x,Y^-)(Y^+-Y^-)
  \longrightarrow Y^+-Y^-
\]
almost surely. Also $0<k_\tau\le1$, so
\[
  \Norm{k_\tau(x,Y^+)k_\tau(x,Y^-)(Y^+-Y^-)}_2
  \le\norm{Y^+}_2+\norm{Y^-}_2.
\]
The right-hand side is integrable under the first-moment assumption.
Dominated convergence gives \eqref{eq:collapse}.
\end{proof}

\begin{proof}[Proof of \cref{cor:normalized-collapse}]
For either $Y\sim p$ or $Y\sim q$, $k_\tau(x,Y)\to1$ almost surely and
$0<k_\tau\le1$, so dominated convergence gives
$\E[k_\tau(x,Y)]\to1$. Independence therefore gives $Z_\tau(x)\to1$.
Combine this with \cref{thm:collapse}.
\end{proof}

\begin{proof}[Proof of \cref{lem:taurate}]
Write
$a:=\norm{x-Y^+}_2^2+\norm{x-Y^-}_2^2\ge0$. By
\eqref{eq:rbf}--\eqref{eq:tau-drift} and the definition of $\Delta\mu$,
\[
  V^{(\tau)}_{p,q}(x)-\Delta\mu
  =\E\big[(e^{-a/\tau}-1)(Y^+-Y^-)\big].
\]
For $u\ge0$, convexity gives $e^{-u}\ge1-u$ and monotonicity gives
$e^{-u}\le1$, hence $\abs{e^{-u}-1}\le u$. Applying this with
$u=a/\tau$ and Jensen's inequality for the Euclidean norm,
\[
  \Norm{V^{(\tau)}_{p,q}(x)-\Delta\mu}_2
  \le\E\big[\abs{e^{-a/\tau}-1}\norm{Y^+-Y^-}_2\big]
  \le\frac1\tau\E\big[a\norm{Y^+-Y^-}_2\big]
  =\frac{C(x,p,q)}{\tau}.
\]
For finiteness, use
$\norm{x-Y}_2^2\le2\norm{x}_2^2+2\norm{Y}_2^2$ and
$\norm{Y^+-Y^-}_2\le\norm{Y^+}_2+\norm{Y^-}_2$. Every resulting term is a
moment of order at most three in one variable or a mixed product of such
moments; mixed terms factor by independence. The assumed third moments
therefore make $C(x,p,q)$ finite.
\end{proof}

\begin{proof}[Proof of \cref{prop:laplace-collapse}]
Write
\[
  a_{\mathrm{Lap}}
  :=\norm{x-Y^+}_2+\norm{x-Y^-}_2.
\]
The Laplace product in the interaction equals
$e^{-a_{\mathrm{Lap}}/\tau}$.  It converges almost surely to one and is
bounded by one, so the dominated-convergence proof of
\cref{thm:collapse} applies unchanged under first moments.  Likewise, each
normalizing expectation converges to one, proving the ratio-normalized
claim.

For the rate, $\abs{e^{-u}-1}\le u$ for $u\ge0$ gives
\[
  \Norm{V^{(\tau),\mathrm{Lap}}_{p,q}(x)-\Delta\mu}_2
  \le\frac1\tau\E\!\left[
    a_{\mathrm{Lap}}\norm{Y^+-Y^-}_2\right].
\]
The displayed expectation is finite under second moments: use
$\norm{x-Y}_2\le\norm{x}_2+\norm{Y}_2$,
$\norm{Y^+-Y^-}_2\le\norm{Y^+}_2+\norm{Y^-}_2$, and independence for the
mixed products.
\end{proof}

\section{Failure Modes}\label{app:failures}

This appendix records an exact, structural failure mode --- a basis
symmetry under which a formal mismatch direction is invisible to the
sampled drift for \emph{every} probe set of \emph{every} size --- and
collects the paper's assumptions and failure modes in \cref{tab:failures}. The example
should be contrasted with the conditioning failures of
\cref{sec:observability,sec:certification} (small but nonzero $\smin$) and the
kernel-degeneracy failure of \cref{sec:collapse} (large bandwidth): here
the deficiency is exact and no probe design repairs it.

\begin{example}[Midpoint collision]\label{ex:midpoint}
Work in $d = 1$ with the Gaussian benchmark model of
\cref{app:randprobes}: basis densities
$\phi_i = \mathcal{N}(\theta_i, \sigma^2)$ with equally spaced means
$\theta = (0, 1, 2, 3)$, Gaussian-RBF kernel \eqref{eq:rbf}, and the
mean-shift interaction \eqref{eq:meanshift}. The pairs $(1, 4)$ and
$(2, 3)$ share the midpoint $m_{14} = m_{23} = 3/2$. By the closed form of
\cref{lem:gaussform},
\[
  g_{14}(x) = c^{\mathrm{gs}}_{14}\, \varphi_{v/2}(x - \tfrac{3}{2})
    (\theta_1 - \theta_4),
  \qquad
  g_{23}(x) = c^{\mathrm{gs}}_{23}\, \varphi_{v/2}(x - \tfrac{3}{2})
    (\theta_2 - \theta_3),
\]
with $c^{\mathrm{gs}}_{14}, c^{\mathrm{gs}}_{23} > 0$, so $g_{14}$ and
$g_{23}$ are proportional \emph{as functions} on $\R$. Consequently there
exists $\lambda \ne 0$ (supported on the coordinates $(1, 4)$ and
$(2, 3)$) with $g_\lambda \equiv 0$, hence $M \lambda = 0$
deterministically for every probe set, $\Gamma(\nu)$ is singular for
every probe law $\nu$, and by \cref{lem:gram}(b) the formal measurement
direction $\lambda$ is invisible. This shows a structural degeneracy of
the kernel/probe instrument. The next proposition shows that this particular
structural null direction is not itself the mismatch vector of a valid pair
of distributions.
\end{example}

\begin{proposition}[The midpoint-collision null line is infeasible]
\label{prop:midpoint-infeasible}
In \cref{ex:midpoint}, let $\lambda$ be the structural null vector
supported on $(1,4)$ and $(2,3)$ with $g_\lambda \equiv 0$. Then no
nonzero multiple of $\lambda$ can be written as $a \wedge b$ for any
$a,b \in \R^4$. Consequently, for any probe set whose nullspace is exactly
the structural line $\operatorname{span}\{\lambda\}$, the midpoint
collision destroys the ambient $\smin(M)$ certificate but does not itself
yield valid $p\ne q$ with zero sampled drift.
\end{proposition}

\begin{proof}
The dependence relation between $g_{14}$ and $g_{23}$ has both coordinates
nonzero. Indeed, using $\theta=(0,1,2,3)$,
\[
  g_{14}(x) = -3 c^{\mathrm{gs}}_{14}
      \varphi_{v/2}(x-3/2),
  \qquad
  g_{23}(x) = -c^{\mathrm{gs}}_{23}
      \varphi_{v/2}(x-3/2),
\]
with $c^{\mathrm{gs}}_{14},c^{\mathrm{gs}}_{23}>0$. Thus a null vector may be
chosen with $\lambda_{14}=1$ and
$\lambda_{23}=-3c^{\mathrm{gs}}_{14}/c^{\mathrm{gs}}_{23}$, so
$\lambda_{14}\lambda_{23}\ne0$.

Now suppose, toward a contradiction, that $t\lambda=a\wedge b$ for some
$t\ne0$ and some $a,b\in\R^4$. The Pl\"ucker relation for a decomposable
$2$-vector in $\Lambda^2\R^4$ is
\[
  c_{12}c_{34}-c_{13}c_{24}+c_{14}c_{23}=0.
\]
For $c=t\lambda$, all coordinates except $(1,4)$ and $(2,3)$ vanish, so the
relation reduces to
\[
  t^2\lambda_{14}\lambda_{23}=0,
\]
contradicting $t\ne0$ and $\lambda_{14}\lambda_{23}\ne0$. Therefore the
structural null line contains no nonzero decomposable mismatch. Since the
feasible mismatch set is contained in the decomposable cone, it contains no
nonzero vector on this null line, proving the stated zero-drift claim.
\end{proof}

The failure is created by the symmetry
$\theta_1 + \theta_4 = \theta_2 + \theta_3$ of the basis means.
\Cref{prop:nondegen} shows that pairwise-distinct midpoints rule out this
obstruction and suffice for linear independence of all pair witnesses; it
does not assert a full converse for every collision geometry.
Near-collisions ($m_{14} \approx m_{23}$) produce the
quantitative analogue: $\gamma(\nu)$ positive but small, with
correspondingly weak certificates (\cref{rem:scaling}).

\begin{table}[H]
\centering
\small
\begin{tabular}{L{3.1cm}L{4.6cm}L{5.4cm}}
\toprule
Assumption or design choice & Role in the analysis & Consequence of violation \\
\midrule
Antisymmetry &
Ensures $p = q \Rightarrow V_{p,q} = 0$ and produces the antisymmetric
coordinates $c_{ij}$. &
Equality may not be an equilibrium, or the reduction to $c_{ij}$ may
fail. \\
Finite-basis representation &
Turns drift identifiability into finite-dimensional linear algebra. &
Finitely many probes need not determine arbitrary distributions. \\
Full column rank of $M$ &
Makes $c \mapsto Mc$ injective on the full formal mismatch space. &
Formal null directions can occur; actual ambiguity requires intersection
with the feasible mismatch set. \\
Well-conditioned $M$ &
Makes approximate zero drift imply small mismatch. &
If $\smin(M)$ is small, small drift can hide large $c$. \\
Distinct means and pair midpoints (Gaussian basis), with
$\nu\ll\Leb^d$ &
Suffice for witness independence and hence $\gamma(\nu) > 0$ for an
absolutely continuous probe law (\cref{prop:nondegen}). &
Certain midpoint collisions with dependent direction vectors make
$\Gamma(\nu)$ singular; actual non-identifiability also requires feasible
coefficients in a resulting null direction (\cref{ex:midpoint},
\cref{sec:feasible}); a discrete probe law can also be rank deficient. \\
Accurate empirical drift &
Connects observed minibatch drift to population drift. &
Sampling noise can create false certificates. \\
Accurate estimated $M$ &
Needed when $M$ is computed by approximation. &
Singular values can be overestimated, making certificates too
optimistic. \\
Good basis approximation &
Needed to transfer finite-basis statements to general densities. &
Residual error $R_m$ may dominate the certificate. \\
Nondegenerate bandwidth &
Keeps the kernel/probe system informative. &
Under finite third moments for the Gaussian kernel, or finite second moments
for the Laplace kernel, large bandwidth collapses toward moment matching at
$O(1/\tau)$ on fixed finite probe sets
(\cref{lem:taurate,prop:laplace-collapse}); small
bandwidth can be noisy or too local. \\
\bottomrule
\end{tabular}
\caption{Main assumptions and failure modes for drift-based
certification.}
\label{tab:failures}
\end{table}

\section{Randomized Probes: Exact Rank and Verification for the Gaussian
Model}\label{app:randprobes}

This appendix complements \cref{sec:randprobe} in two directions. First,
it shows that under real-analyticity, i.i.d.\ probes from any absolutely
continuous law achieve, almost surely, the best rank attainable by
\emph{any} deterministic probe design of the same size
(\cref{prop:analytic}). Second, it verifies the hypotheses ---
analyticity and the population observability condition (A10) --- for the
paper's concrete benchmark model
(\cref{lem:gaussform,prop:nondegen}), so that \cref{thm:randprobe}
applies there with analytically verified hypotheses and an explicit (A9)
bound. The value of $\gamma(\nu)$ still depends on the chosen probe law.

\begin{lemma}[Greedy probe existence]\label{lem:greedy}
If $\{g_\alpha\}_{\alpha \in P}$ are linearly independent as functions on
$\R^d$, then there exist $N_0 \le r$ points $x_1, \dots, x_{N_0}$ such
that $M(x_1, \dots, x_{N_0})$ has full column rank $r$.
\end{lemma}

\begin{proof}
Let $\mathcal{V}_k := \{\lambda \in \R^r : g_\lambda(x_\ell) = 0,\
\ell = 1, \dots, k\}$ denote the nullspace of the current observation
matrix, with $\mathcal{V}_0 = \R^r$. If $\dim \mathcal{V}_k \ge 1$,
choose $\lambda \in \mathcal{V}_k \setminus \{0\}$; by linear
independence $g_\lambda \not\equiv 0$, so some $x_{k+1}$ has
$g_\lambda(x_{k+1}) \ne 0$, whence
$\lambda \notin \mathcal{V}_{k+1} \subseteq \mathcal{V}_k$ and the
dimension strictly decreases. After at most $r$ steps
$\mathcal{V}_{N_0} = \{0\}$, i.e.\ full column rank.
\end{proof}

\begin{proposition}[Random probes achieve the maximal rank]
\label{prop:analytic}
Assume each $g_\alpha$ is real-analytic on $\R^d$, the family
$\{g_\alpha\}_{\alpha \in P}$ is linearly independent as functions, and
$\nu$ is absolutely continuous with respect to Lebesgue measure. Let
$X_1, \dots, X_N \overset{\mathrm{iid}}{\sim} \nu$, and define
\[
  \rho_N := \max_{(x_1,\dots,x_N)\in\R^{dN}}
    \rk M(x_1,\dots,x_N).
\]
Then:
\begin{itemize}
\item[(a)] $\rk M(X_1, \dots, X_N) = \rho_N$ almost surely;
\item[(b)] $\rho_N$ is nondecreasing in $N$ and $\rho_N = r$ for all
$N \ge N_0$ with $N_0 \le r$ as in \cref{lem:greedy}; in particular, for
$N \ge r$, $M$ has full column rank almost surely.
\end{itemize}
\end{proposition}

\begin{proof}
(a) Fix $N$ and set $k := \rho_N$. Let $F_k(x_1, \dots, x_N)$ be the sum
of squares of all $k \times k$ minors of $M$. Each minor is a polynomial
in the entries of $M$; each entry is a coordinate of some
$g_\alpha(x_\ell)$; and sums, products, and compositions of real-analytic
functions are real-analytic. Hence $F_k$ is real-analytic on $\R^{dN}$.
By the definition of rank, $\rk(M) \ge k \iff F_k > 0$; since some
configuration attains rank $\rho_N = k$, $F_k \not\equiv 0$. The zero set
of a real-analytic function on $\R^{dN}$ that is not identically zero has
Lebesgue measure zero \citep{mityagin2015zero}. Since
$\nu \ll \Leb^d$ implies $\nu^{\otimes N} \ll \Leb^{dN}$, the event
$\{F_k = 0\}$ has probability zero, so almost surely
$\rk(M) \ge \rho_N$; the reverse inequality holds surely by maximality of
$\rho_N$.

(b) Appending a probe appends rows to $M$, which cannot decrease the
rank, so $\rho_N$ is nondecreasing; padding the configuration of
\cref{lem:greedy} with arbitrary extra points shows $\rho_N = r$ for
$N \ge N_0 \le r$.
\end{proof}

We now verify the hypotheses for the Gaussian benchmark model: mean-shift
interaction \eqref{eq:meanshift}, Gaussian-RBF similarity \eqref{eq:rbf},
and Gaussian-mixture basis. Throughout, $\varphi_v$ denotes the
$\mathcal{N}(0, v I_d)$ density and $m_{ij} := (\theta_i + \theta_j)/2$.

\begin{lemma}[Closed form and analyticity]\label{lem:gaussform}
Let $\phi_i$ be the $\mathcal{N}(\theta_i, \sigma^2 I_d)$ density,
$i = 1, \dots, m$. Set $v := \sigma^2 + \tau/2$ and
$\kappa := \tau / (2\sigma^2 + \tau) \in (0, 1)$. Then for all
$(i, j) \in P$ and all $x \in \R^d$,
\[
  g_{ij}(x)
  = c^{\mathrm{gs}}_{ij}\; \varphi_{v/2}(x - m_{ij})\;
    (\theta_i - \theta_j),
  \qquad
  c^{\mathrm{gs}}_{ij}
  := \kappa\, (\pi\tau)^d\, \varphi_{2v}(\theta_i - \theta_j) > 0,
\]
and each $g_{ij}$ is real-analytic on $\R^d$.
\end{lemma}

\begin{proof}
\emph{Step 1 (factorization of the witness).}
Substituting \eqref{eq:meanshift} into \eqref{eq:pairwitness} and
separating the product measure (Fubini; the integrand is absolutely
integrable since $k_\tau \le 1$ and each $\phi_i$ is a Gaussian density
with finite first moment),
\[
  g_{ij}(x) = m_i(x)\, s_j(x) - s_i(x)\, m_j(x),
  \quad
  s_i(x) := \!\int\! k_\tau(x, y) \phi_i(y) \dd y,
  \quad
  m_i(x) := \!\int\! k_\tau(x, y)\, y\, \phi_i(y) \dd y.
\]

\emph{Step 2 (tilted mean).}
The density proportional to $k_\tau(x, y) \phi_i(y)$ has $y$-exponent
$-\norm{y}_2^2 (1/\tau + 1/(2\sigma^2)) + \ip{y}{2x/\tau +
\theta_i/\sigma^2} + \text{const}$, a nondegenerate quadratic; completing
the square, it is Gaussian with mean
$\mu_i(x) = \big( 2x/\tau + \theta_i/\sigma^2 \big) \big/
\big( 2/\tau + 1/\sigma^2 \big)$. Hence $m_i = s_i\, \mu_i$ and
\[
  \mu_i(x) - \mu_j(x)
  = \frac{(\theta_i - \theta_j)/\sigma^2}{2/\tau + 1/\sigma^2}
  = \kappa\, (\theta_i - \theta_j),
\]
a constant vector, so
$g_{ij}(x) = s_i(x)\, s_j(x)\, (\mu_i(x) - \mu_j(x))
= \kappa\, s_i(x)\, s_j(x)\, (\theta_i - \theta_j)$.

\emph{Step 3 (Gaussian convolution and product).}
Writing $e^{-\norm{u}_2^2/\tau} = (\pi\tau)^{d/2} \varphi_{\tau/2}(u)$
(match of normalizing constants: $(2\pi \cdot \tau/2)^{d/2} =
(\pi\tau)^{d/2}$), convolution of Gaussians gives
$s_i(x) = (\pi\tau)^{d/2}\, (\varphi_{\tau/2} * \varphi_{\sigma^2})(x -
\theta_i) = (\pi\tau)^{d/2} \varphi_v(x - \theta_i)$ with
$v = \sigma^2 + \tau/2$ (variances add under convolution). The
equal-covariance Gaussian product identity
\[
  \varphi_v(x - \theta_i)\, \varphi_v(x - \theta_j)
  = \varphi_{2v}(\theta_i - \theta_j)\; \varphi_{v/2}(x - m_{ij}),
\]
verified by completing the square in the exponent
(coordinatewise, $\tfrac{(x - \theta_i)^2 + (x - \theta_j)^2}{2v}
= \tfrac{(x - m_{ij})^2}{v} + \tfrac{(\theta_i - \theta_j)^2}{4v}$) and
matching normalizers ($(2\pi v)^{-d} = (4\pi v)^{-d/2} (\pi v)^{-d/2}$),
yields the claimed closed form.

\emph{Step 4 (analyticity).}
$\varphi_{v/2}(x - m_{ij})$ is the composition of the entire function
$\exp$ with a polynomial, hence real-analytic; multiplication by the
constant vector $c^{\mathrm{gs}}_{ij} (\theta_i - \theta_j)$ preserves
analyticity.
\end{proof}

\begin{proposition}[Population observability for Gaussian-mixture bases]
\label{prop:nondegen}
In the setting of \cref{lem:gaussform}, assume the means
$\theta_1, \dots, \theta_m$ are pairwise distinct and the midpoints
$\{m_{ij}\}_{(i,j) \in P}$ are pairwise distinct. Then
$\{g_{ij}\}_{(i,j) \in P}$ are linearly independent as functions on
$\R^d$; consequently $\gamma(\nu) > 0$ for \emph{every} probe law
$\nu \ll \Leb^d$, i.e.\ (A10) holds, and \cref{prop:analytic} applies.
Moreover (A9) holds with
$L^2 = \sum_{(i,j) \in P} (c^{\mathrm{gs}}_{ij})^2\,
\varphi_{v/2}(0)^2\, \norm{\theta_i - \theta_j}_2^2$.
\end{proposition}

\begin{proof}
\emph{Independence.} Suppose
$\sum_{(i,j) \in P} \lambda_{ij}\, g_{ij} \equiv 0$ on $\R^d$. Fix a
coordinate $k \in \{1, \dots, d\}$; by \cref{lem:gaussform} the $k$-th
component reads
\[
  \sum_{(i,j) \in P} \lambda_{ij}\, c^{\mathrm{gs}}_{ij}\,
  (\theta_i - \theta_j)_k\; \varphi_{v/2}(x - m_{ij}) \;\equiv\; 0.
\]
Equal-width Gaussians with pairwise distinct centers are linearly
independent: multiplying by $e^{\norm{x}_2^2 / v}$ reduces the identity
to $\sum_s \beta_s\, e^{\ip{x}{b_s}} \equiv 0$ with pairwise distinct
$b_s = 2 m_{\alpha_s} / v$; choose $u \in \R^d$ such that the inner
products $\ip{u}{b_s}$ are pairwise distinct (each pairwise equality
excludes a hyperplane of $u$'s, and finitely many hyperplanes cannot
cover $\R^d$); restricting to the ray $x = t u$, $t \to \infty$, gives
finitely many univariate exponentials $e^{t \ip{u}{b_s}}$ with distinct
rates, which are linearly independent by the dominant-growth argument.
Hence every coefficient vanishes:
$\lambda_{ij}\, c^{\mathrm{gs}}_{ij}\, (\theta_i - \theta_j)_k = 0$ for
all $(i, j)$ and all $k$, i.e.\
$\lambda_{ij}\, c^{\mathrm{gs}}_{ij}\, (\theta_i - \theta_j) = 0$; since
$c^{\mathrm{gs}}_{ij} > 0$ and $\theta_i \ne \theta_j$,
$\lambda_{ij} = 0$.

\emph{(A10) for all $\nu \ll \Leb$.} If $g_\lambda = 0$ $\nu$-a.e.\ with
$\lambda \ne 0$, then $g_\lambda$ is real-analytic
(\cref{lem:gaussform}; sums of analytic functions) and $\not\equiv 0$
(independence), so its zero set is Lebesgue-null
\citep{mityagin2015zero}, contradicting $\nu$-a.e.\ vanishing. By
\cref{lem:gram}(a), $\gamma(\nu) > 0$.

\emph{(A9).} From the closed form,
$\norm{g_{ij}(x)}_2 \le c^{\mathrm{gs}}_{ij}\, \varphi_{v/2}(0)\,
\norm{\theta_i - \theta_j}_2$ for every $x$ (the Gaussian density is
maximized at its mode), and
$\norm{G(x)}_{\mathrm{op}} \le \norm{G(x)}_F$ gives the stated $L$.
\end{proof}

\begin{proposition}[Directional rank obstruction]
\label{prop:directional-rank}
In the setting of \cref{lem:gaussform}, let
\[
  d_\theta := \dim \operatorname{span}
  \{\theta_i-\theta_j : (i,j)\in P\}.
\]
Then for every deterministic probe set of size $N$,
\[
  \rk M(x_1,\dots,x_N) \le \min\{d_\theta N,r\}.
\]
\end{proposition}

\begin{proof}
For a single probe $x$, let $G(x)\in\R^{d\times r}$ be the block whose
columns are $g_{ij}(x)$. By \cref{lem:gaussform}, every column of $G(x)$ is
a scalar multiple of some difference vector $\theta_i-\theta_j$; hence
$\operatorname{col}(G(x))\subseteq
\operatorname{span}\{\theta_i-\theta_j:(i,j)\in P\}$ and
$\rk G(x)\le d_\theta$. The full observation matrix is the vertical stack
of $G(x_1),\dots,G(x_N)$, so subadditivity of rank under vertical stacking
gives
\[
  \rk M(x_1,\dots,x_N)
  \le \sum_{\ell=1}^N \rk G(x_\ell)
  \le d_\theta N.
\]
The bound by $r$ is the column-count bound.
\end{proof}

\begin{conjecture}[Ambient full-rank probe count with directional span]
\label{conj:minprobes}
In the setting of \cref{prop:nondegen}, with $d_\theta$ as in
\cref{prop:directional-rank},
\[
  \rho_N = \min\{d_\theta N,r\}
  \qquad\text{for every }N\ge1.
\]
Equivalently, each additional probe contributes the maximal number
$d_\theta$ of independent directions allowed by the span of the Gaussian
mean differences, until full column rank is reached.
\end{conjecture}

\begin{remark}[Directional-span obstruction]
\label{rem:decidable}
The ambient count $dN\ge r$ can be too optimistic when the difference vectors
span a proper subspace of $\R^d$. If the means are collinear in $\R^2$, then
$d_\theta=1$, and \cref{prop:directional-rank} gives the sharper hard bound
$\rk M\le N$ even though $dN=2N$. For instance, with
$\theta_i=t_i e_1$ and $t=(0,1,2,4.5)$, the means and all pair midpoints are
distinct, but the directional bound still gives $\rk M\le N$. Thus the naive
ambient dimension count $dN\ge r$ would permit full rank already at $N=3$
when $r=6$, whereas the directional obstruction rules it out.

By \cref{prop:analytic}(a), for any fixed
$(m,d,\tau,\sigma^2,\theta)$ the rank of $M$ at a single random probe
configuration of size $N$ equals $\rho_N$ almost surely; numerical rank
computations therefore probe \cref{conj:minprobes} directly for that
configuration size; see \cref{sec:exp-rank}. \Cref{lem:greedy} guarantees
full column rank for $N\ge r$. The open content is the intermediate regime,
especially $2\le d_\theta < r$ and $N<r$. This is a conjecture about
ambient recovery of every vector in $\R^r$; because feasible mismatches
lie on the structured set of \cref{sec:feasible}, it is not a conjecture
about the intrinsic minimum number of probes needed for distributional
identifiability.
\end{remark}

\section{Additional Experimental Diagnostics}\label{app:diagnostics}

\subsection{Cross-bandwidth calibration and exact mismatch}
\label{app:calibration-mismatch}

The cross-bandwidth study in \cref{sec:exp-calibration} compares the
empirical raw drift and radius-free conditioned plug-in with sliced $W_2$.
Both observation designs are tall and numerically full-column-rank, but
conditioning does not make either scale comparable across bandwidths.

\begin{figure}[H]
  \centering
  \begin{minipage}{0.48\textwidth}
    \centering
    \includegraphics[width=\linewidth]{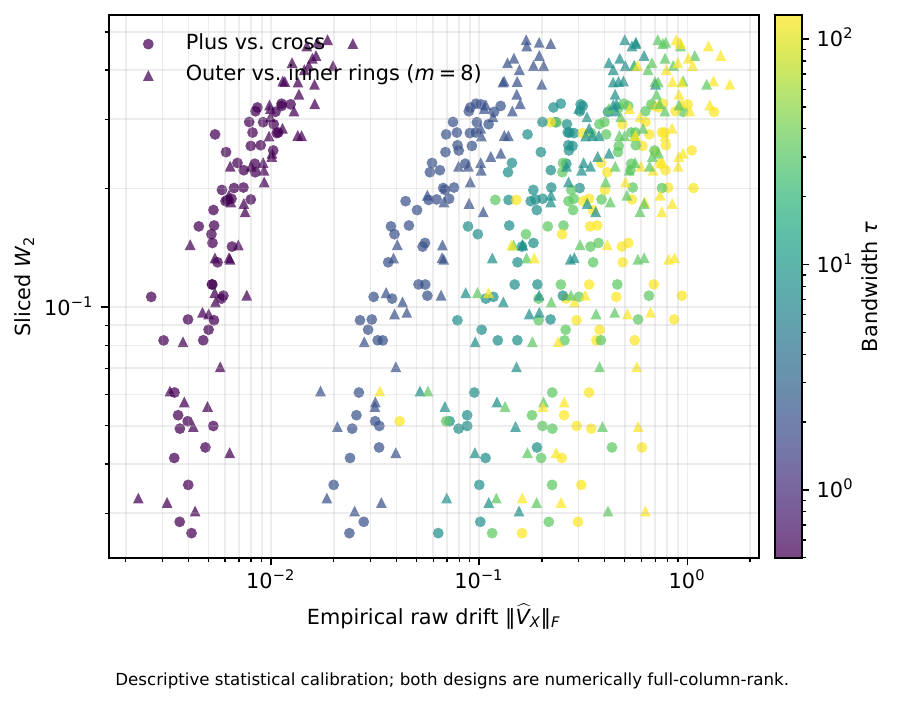}
  \end{minipage}\hfill
  \begin{minipage}{0.48\textwidth}
    \centering
    \includegraphics[width=\linewidth]{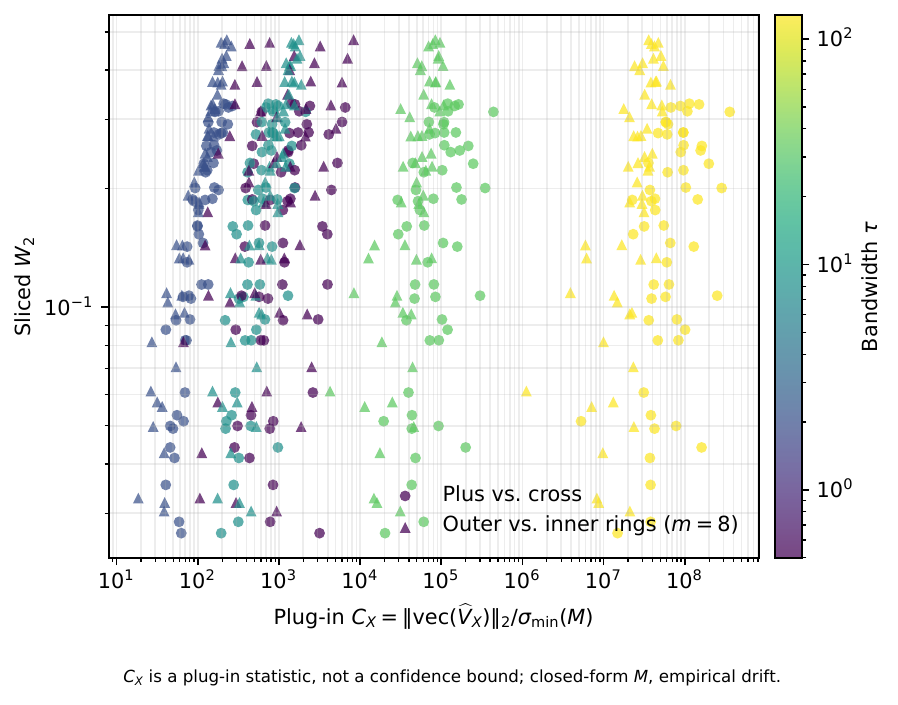}
  \end{minipage}
  \caption{Cross-bandwidth calibration against sliced $W_2$. Left: empirical
  raw drift. Right: the radius-free plug-in $C_X$. Points are statistical;
  $M$ is closed form and its singular values are evaluated numerically.
  In these benchmarks, conditioning is essential for an inverse bound but
  does not create a bandwidth-invariant distance scale.}
  \label{fig:calibration-sw}
\end{figure}

The radius-free statistic $C_X$ from \eqref{eq:CX} can be extremely large
when the worst observed mismatch direction is nearly singular. The comparison
below is descriptive only: it omits the drift-error radius and therefore has
no coverage claim.

\begin{figure}[H]
  \centering
  \includegraphics[width=0.62\textwidth]{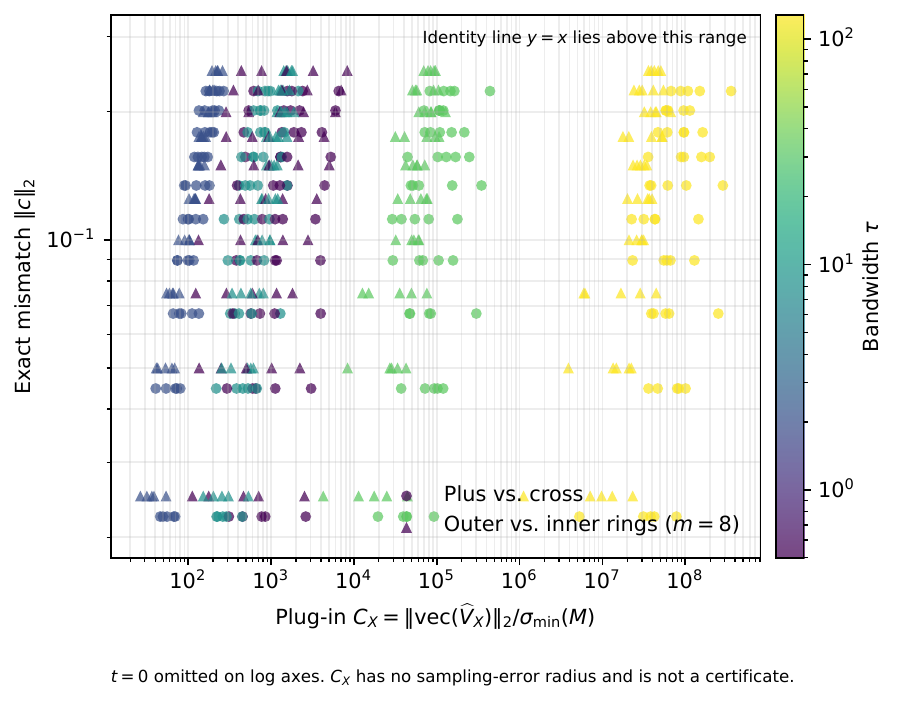}
  \caption{Radius-free plug-in $C_X$ versus exact finite-basis mismatch.
  The identity line lies above the displayed vertical range because the
  plotted plug-in values are inflated by small numerical singular values.}
  \label{fig:calibration-c}
\end{figure}

\subsection{Ambient numerical-rank breakdown}\label{app:rank-breakdown}

\Cref{tab:rank-summary} gives the detailed double-precision outcomes for the
ambient-rank study summarized in \cref{sec:exp-rank}. Boundary ratios are
reported rather than converted into exact rank claims; all $40$ ambiguous or
threshold-mismatched boundary rows are recomputed at $60$ and $100$ decimal
digits as described in the main text.
Each matrix is column normalized before SVD, and
$\rho:=\min\{dN,r\}$.  The table reports
$\sigma_\rho/\sigma_1$, not $\sigma_r/\sigma_1$ when $dN<r$.

\begin{table}[H]
\centering
\small
\begin{tabular}{@{}rrrrrrrcc@{}}
\toprule
$m$ & $d$ & $r$ & rows & match & ambig. & mismatch & min $\sigma_\rho/\sigma_1$ & max $\sigma_\rho/\sigma_1$ \\
\midrule
4 & 1 & 6 & 40 & 37 & 3 & 0 & $9.56\times 10^{-12}$ & $1.00$ \\
4 & 2 & 6 & 25 & 23 & 1 & 1 & $7.48\times 10^{-18}$ & $8.58\times 10^{-1}$ \\
5 & 1 & 10 & 60 & 48 & 12 & 0 & $1.34\times 10^{-14}$ & $1.00$ \\
6 & 2 & 15 & 50 & 45 & 5 & 0 & $2.43\times 10^{-9}$ & $9.01\times 10^{-1}$ \\
6 & 3 & 15 & 35 & 32 & 2 & 1 & $2.49\times 10^{-15}$ & $5.93\times 10^{-1}$ \\
8 & 2 & 28 & 80 & 65 & 14 & 1 & $7.42\times 10^{-16}$ & $8.83\times 10^{-1}$ \\
\bottomrule
\end{tabular}
\caption{Double-precision ambient-rank summary after column normalization,
where $\rho=\min\{dN,r\}$.  ``Ambiguous'' means
$10^{-12}\le\sigma_\rho/\sigma_1\le10^{-8}$.  Outside that band, numerical
rank uses NumPy's default tolerance
$\sigma_1\max\{dN,r\}\epsilon_{\rm mach}$.}
\label{tab:rank-summary}
\end{table}

\end{document}